\documentclass[11pt,letterpaper]{article}

\usepackage[margin=1in]{geometry}
\usepackage[hyphens]{url}
\usepackage{graphicx}
\usepackage{microtype}
\usepackage{booktabs}
\usepackage{multirow}
\usepackage{algorithm}
\usepackage{algorithmic}
\usepackage{amsmath}

\usepackage{amssymb}
\usepackage{mathtools}
\usepackage{amsthm}
\usepackage{natbib}
\usepackage{caption}
\usepackage{placeins}
\usepackage{float}
\usepackage[hidelinks]{hyperref}
\hypersetup{pdftitle={MOON: Multi-Objective OrthoNormalized Updates for Multitask Learning},pdfauthor={Shiji Zhou, Kunlin Lyu, Lei Zhang, Ruodong Wang, Yifan Sun}}
\usepackage[capitalize,noabbrev]{cleveref}
\providecommand{\Description}[1]{}

\theoremstyle{plain}
\newtheorem{theorem}{Theorem}
\newtheorem{proposition}{Proposition}
\newtheorem{lemma}{Lemma}
\newtheorem{corollary}{Corollary}

\theoremstyle{definition}
\newtheorem{definition}{Definition}
\newtheorem{assumption}{Assumption}

\theoremstyle{remark}

\def\argmin{\mathop{\arg\min}}

\AtBeginDocument{}
\title{MOON: Multi-Objective OrthoNormalized Updates for Multitask Learning}
\author{
Shiji Zhou$^{1,2,3}$ \quad Kunlin Lyu$^{4}$ \quad Lei Zhang$^{1}$ \quad 
Ruodong Wang$^{4}$ \quad Yifan Sun$^{2,4,*}$\\[0.5em]
\small $^{1}$Institute of Artificial Intelligence, Beihang University\\
\small $^{2}$Beijing Advanced Innovation Center for Future Blockchain and Privacy Computing,\\\small Beihang University\\
\small $^{3}$Beijing Academy of Artificial Intelligence (BAAI)\\
\small $^{4}$Center for Applied Statistics, School of Statistics, Renmin University of China\\
\small $^{*}$Corresponding author
}
\date{}

\begin{document}

\maketitle

\begin{abstract}
Multi-objective optimization (MOO) has demonstrated significant success in multi-task learning by mitigating task conflicts through gradient manipulation. However, most existing methods flatten model parameters into vectors and perform gradient manipulation under Euclidean geometry, thereby overlooking the matrix structure prevalent in modern architectures such as Transformers. In this paper, we show that gradient manipulation in Euclidean space does not generally yield the steepest descent direction under matrix geometry, potentially limiting optimization efficiency. Drawing from the theory of steepest descent for matrix-valued parameters, we propose MOON (Multi-Objective OrthoNormalized Updates), which performs gradient manipulation under spectral--nuclear norm geometry and uses the orthonormalized manipulated gradient for parameter updates. Theoretically, for smooth non-convex objectives, we establish convergence of the averaged Pareto-stationarity measure at rates of $\mathcal{O}(T^{-1/2})$ in the deterministic setting and $\mathcal{O}(T^{-1/4})$ under stochastic gradients. Empirical results across various benchmarks show that MOON consistently improves both optimization efficiency and final multi-task performance. Our code is available at \url{https://github.com/KunlinLyu/MOON}.
\end{abstract}

\section{Introduction}
\label{sec:intro}

A wide range of practical problems in machine learning are inherently multi-objective, where the objectives often conflict and must be traded off \citep{sener2018multi,peitz2025multi}. In deep learning, task conflicts often manifest as gradient conflicts, with task gradients exhibiting competing directions during optimization. Multi-objective optimization (MOO) methods \citep{zitzler1999multiobjective} mitigate task-gradient conflicts via gradient manipulation and have achieved great success in multi-task learning \citep{sener2018multi,yu2020gradient,navon2022multi,liu2024famo}.

Existing MOO methods such as MGDA \citep{sener2018multi}, PCGrad \citep{yu2020gradient}, FAMO \citep{liu2024famo}, and other variants \citep{liu2021conflict,liu2021towards,navon2022multi} typically flatten model parameters into a single vector and then perform gradient manipulation in this Euclidean vector space. However, modern architectures, such as Transformers \citep{vaswani2017attention}, are dominated by matrix-valued parameters. Blindly flattening all parameters into vectors ignores the linear-mapping structure of matrix-valued parameters and may fail to identify the steepest descent direction under matrix geometry. This motivates MOO methods tailored to network architectures with matrix-valued parameters.

One possible approach is to combine conventional MOO methods with matrix-aware optimizers such as Muon \citep{jordan2024muon}, and MuonClip \citep{team2025kimi}, which exploit matrix structure through preconditioning or orthonormalized updates. However, conventional MOO manipulates gradients under Euclidean geometry, whereas matrix-aware optimizers apply structure-preserving transformations induced by matrix geometry. Naively combining them, therefore, lacks a consistent geometry and creates a theoretical gap between their underlying assumptions, creates a mismatch between their underlying optimization principles. Moreover, establishing MOO convergence under this matrix geometry is nontrivial. Since orthonormalization discards the singular-value magnitudes of the aggregated gradient, standard Euclidean descent analyses based on squared gradient norms are no longer directly applicable. The analysis must instead characterize Pareto stationarity under matrix geometry and exploit spectral--nuclear norm duality to connect the orthonormalized direction with optimization progress.

To address these challenges, we propose Multi-Objective OrthoNormalized Updates (MOON), a geometrically consistent framework for multi-objective optimization with matrix-valued parameters \citep{bernstein2024old}. Starting from simultaneous matrix-smooth upper bounds of the objectives, we formulate the selection of a common descent direction as a spectral-norm-regularized minimax problem. Through spectral--nuclear norm duality, the corresponding dual problem determines the task weights by minimizing the nuclear norm of the weighted aggregate gradient, while the solution to the primal problem is given by its polar factor. This formulation provides a principled connection between multi-objective gradient manipulation and orthonormalized matrix updates, distinguishing MOON from conventional MOO methods based on Euclidean geometry. For smooth non-convex objectives, we establish convergence of the averaged Pareto-stationarity measure at rates of $\mathcal{O}(T^{-1/2})$ under deterministic gradients and $\mathcal{O}(T^{-1/4})$ under unbiased stochastic gradients. Extensive experiments further demonstrate that MOON accelerates optimization and achieves competitive or improved final performance across multiple benchmarks.

Our main contributions are: \textbf{(1)} We propose MOON, a structure-aware MOO method. MOON is derived by our analysis of multi-objective steepest descent for matrix values, providing a geometrically consistent update for modern architectures. \textbf{(2)} We establish convergence to Pareto stationarity for smooth non-convex objectives, where MOON attains an $O(T^{-1/2})$ rate under exact gradients and an $O(T^{-1/4})$ rate under unbiased stochastic gradients with bounded noise. \textbf{(3)} We conduct extensive experiments on diverse multi-task scenarios with various architectures, demonstrating faster optimization and competitive or improved final performance compared with representative MOO baselines.

\section{Background}
\label{sec:prelim}
\paragraph{Multi-Objective Optimization.}
Multi-objective optimization (MOO) aims to optimize multiple objectives simultaneously \citep{zitzler1999multiobjective}:
\begin{equation}\label{eq:moo_problem}
    \min_{\boldsymbol{\theta}} \boldsymbol{L}(\boldsymbol{\theta}) = (\ell^1(\boldsymbol{\theta}),\dots,\ell^m(\boldsymbol{\theta}))^\top,
\end{equation}
where $m \ge 2$ denotes the number of objectives, and $\ell^i:\mathbb{R}^n\rightarrow \mathbb{R}$ is the $i$-th loss function. Denote $\Delta_m$ as the $(m-1)$-dimensional probability simplex. The concept of Pareto optimality/stationarity~\citep{custodio2011direct} is introduced to determine whether a solution to MOO is optimal/critical. 
\begin{definition}\label{def::partial_order}\textbf{Pareto Optimality:}
    For any two solutions $\boldsymbol{\theta}, \boldsymbol{\theta}'$, we say that $\boldsymbol{\theta}$ dominates $\boldsymbol{\theta}'$, denoted as $\boldsymbol{\theta} \prec \boldsymbol{\theta}'$ or $\boldsymbol{\theta}'\succ \boldsymbol{\theta}$, if $\ell^i(\boldsymbol{\theta}) \leq \ell^i(\boldsymbol{\theta}')$ for all $i$, and there exists an $i$ such that $\ell^i(\boldsymbol{\boldsymbol{\theta}}) < \ell^i(\boldsymbol{\theta}')$. A solution $\boldsymbol{\theta}^\ast$ is called Pareto optimal if it is not dominated by any other solution. \textbf{Pareto Stationary:} A solution $\boldsymbol{\theta}$ is called Pareto stationary if there exists some $\boldsymbol{\lambda} \in \Delta_{m}$ such that $\sum_{i=1}^m \lambda_i \nabla_{\boldsymbol{\theta}} \ell^i(\boldsymbol{\theta})=0$.
\end{definition}

The typical \textbf{Multiple Gradient Descent Algorithm (MGDA)} method~\citep{sener2018multi} has been shown to be the \textbf{steepest method} for MOO~\cite{fliege2000steepest}. Specifically, 
under the assumption of $L$-smoothness with respect to the $\|\cdot\|_2$ norm, in each iteration $t$, MGDA aims to find a direction $\boldsymbol{d}_t$ to maximize the minimum decrease in losses by solving the following subproblem:
\begin{equation}\label{eq:MGDA}
    \boldsymbol{d}_t = \argmin_{\boldsymbol{d}\in \mathbb{R}^d} \max_{i\in[m]} \{ \nabla \ell^i(\boldsymbol{\theta}_t)^\top \boldsymbol{d} + \frac{L}{2}\|\boldsymbol{d}\|^2_2\}.
\end{equation}
The right-hand side $\nabla \ell^i(\boldsymbol{\theta}_t)^\top \boldsymbol{d} + \frac{L}{2}\|\boldsymbol{d}\|^2_2$ is actually the quadratic upper bound of the objective decrease $\ell^i (\boldsymbol{\theta}_t + \boldsymbol{d}) - \ell^i (\boldsymbol{\theta}_t)$. MOO aims to optimize all the objectives simultaneously, and mathematically to maximize the minimum objective decrease, which formulates the subproblem above. Other MOO methods such as PCGrad~\cite{yu2020gradient}, CAGrad~\cite{liu2021conflict} and other variants~\cite{zhou2022convergence,navon2022multi,fernando2023mitigating,he2024robust} also share a similar idea of seeking a mutual descent direction by optimizing towards some specific upper bounds. 

\paragraph{Steepest Descent via Orthonormalized Updates.}
\label{subsec:muon}
In modern neural network architectures, such as Transformer, matrix parameters are fundamental components. 

Suppose the objective $\mathcal{L}:\mathbb{R}^{p\times q}\rightarrow \mathbb{R}$ is $L$-smooth with respect to the spectral norm $\| \cdot \|_{S_\infty}$. Then $\mathcal{L}$ admits an upper bound that is quadratic in the spectral norm\citep{bernstein2024old}. 
\begin{equation}
    \mathcal{L}(\boldsymbol{\Theta}) \leq \mathcal{L}(\boldsymbol{\Theta}') + \langle \nabla \mathcal{L}(\boldsymbol{\Theta}'), \boldsymbol{\Theta} - \boldsymbol{\Theta}' \rangle + \frac{L}{2} \| \boldsymbol{\Theta} - \boldsymbol{\Theta}' \|_{S_{\infty}}^{2}.
    \label{eq:upperbound}
\end{equation}

By minimizing this quadratic upper bound, we obtain the steepest descent under the spectral norm. 

\begin{proposition}[Steepest Descent for Matrix Parameters \cite{bernstein2024old}]
\label{prop:steepest_spec}
In each iteration $t$,  steepest descent under the spectral norm aims to solve the following subproblem: 
\begin{equation}\label{eq:spectral_SGD}
    \Delta \boldsymbol{\Theta}_t =  \argmin_{\Delta \boldsymbol{\Theta}} \left[ \langle \nabla \mathcal{L}(\boldsymbol{\Theta}_t), \Delta \boldsymbol{\Theta} \rangle + \frac{L}{2} \| \Delta \boldsymbol{\Theta} \|_{S_{\infty}}^2 \right],
\end{equation}
where $\langle \cdot, \cdot \rangle$ denotes the Frobenius inner product.
\end{proposition} 
From \citet{bernstein2024old} we know that, if $\nabla \mathcal{L}(\boldsymbol{\Theta}_t)$ has a reduced SVD given by $\nabla \mathcal{L}(\boldsymbol{\Theta}_t) = \boldsymbol{U}_t \boldsymbol{\Sigma}_t \boldsymbol{V}_t^{\top}$, the subproblem \eqref{eq:spectral_SGD} is solved with a step size $\alpha_t = \text{Tr}(\boldsymbol{\Sigma}_t) = \| \nabla \mathcal{L}(\boldsymbol{\Theta}_t) \|_{S_1}/L$ and an update: 
\begin{equation}\label{eq:spectral_SGD_update}
    \Delta \boldsymbol{\Theta}_t = -\alpha_t \cdot \boldsymbol{U}_t \boldsymbol{V}_t^{\top}.
\end{equation}

This result highlights a fundamental difference between steepest-descent updates for matrix-valued parameters and those derived under Euclidean vector geometry. Whereas the Euclidean gradient preserves the original singular-value magnitudes, the spectral-norm steepest-descent direction is characterized by the polar factor $\boldsymbol{U}_t\boldsymbol{V}_t^\top$, which preserves the singular subspaces while normalizing the nonzero singular values. Consequently, the resulting direction is less dominated by a few large singular components, while its overall update magnitude can be controlled separately by the learning rate.

Matrix-aware optimizers such as Shampoo \citep{gupta2018shampoo} and Muon \citep{jordan2024muon} have recently attracted increasing attention for exploiting matrix structure during optimization. Shampoo constructs structured preconditioners from gradient statistics, whereas Muon applies Newton--Schulz iterations to a momentum-smoothed gradient matrix to efficiently approximate its polar factor without an explicit singular value decomposition. Empirical studies have demonstrated the effectiveness and optimization efficiency of this strategy in large-scale Transformer training in large language models \citep{liu2025muon,team2025kimi}.

\section{Multi-Objective Orthonormalized Updates}\label{sec:method}
Although solving the Euclidean-norm minimax problem (Equation \ref{eq:MGDA}) yields the steepest descent direction in the Euclidean space, this direction may not be the steepest descent direction in the native matrix space of the matrix-valued parameters—such as the formulation shown in Equation \ref{eq:spectral_SGD}. Consequently, existing methods struggle to adapt to the training of modern neural network architectures. Therefore, we extend the multi-objective optimization framework to matrix-valued parameters. For clarity, we present the derivation using a single matrix-valued parameter. For networks containing multiple matrix-valued parameter blocks, the same construction is applied blockwise.

Suppose that for each objective $\ell^{i}: \mathbb{R}^{p\times q}\rightarrow \mathbb{R}$ that is $L$-smooth with respect to the spectral norm $\| \cdot \|_{S_\infty}$, we assume that an update $\boldsymbol{\Theta}_{t+1} = \boldsymbol{\Theta}_{t}-\alpha\boldsymbol{W}_t$ is performed at step $t$ and derive its quadratic upper bound from Equation (\ref{eq:upperbound})
\begin{equation}
    \ell^{i}(\boldsymbol{\Theta}_{t+1})\leq \ell^{i}(\boldsymbol{\Theta}_{t}) - \alpha\langle\nabla\ell^{i}(\boldsymbol{\Theta}_{t}), \boldsymbol{W}_t\rangle + \frac{L\alpha^2}{2}\|\boldsymbol{W}_t\|_{S_{\infty}}^{2}.
\end{equation}

Similar to Equation \ref{eq:MGDA}, we seek a mutual update direction in the matrix space that minimizes the worst-case upper bound. Thus, we aim to solve the following problem.
\begin{proposition}
    Let $\alpha \leq 1/L$. Minimizing the common upper bound is equivalent to solving the following spectral-norm-regularized minimax problem:
    \begin{equation}
        \min_{\boldsymbol{W}_t}\max_{i} - \langle\nabla\ell^{i}(\boldsymbol{\Theta}_{t}), \boldsymbol{W}_t\rangle + \frac{1}{2}\|\boldsymbol{W}_t\|_{S_{\infty}}^{2}.
        \label{eq:moon_primal}
    \end{equation}
\end{proposition}
Consequently, we derive the quadratic upper bound for MOO with matrix parameters, which distinguishes our approach from traditional MOO methods.

Since the primal problem (\ref{eq:moon_primal}) is difficult to solve, we usually consider the following dual problem. 

\begin{proposition}
    The dual problem of (\ref{eq:moon_primal}) is
    \begin{equation}
        \min_{\boldsymbol{z}\in\Delta_{m}} \quad \frac{1}{2}\|\sum_{i=1}^{m} z_{i,t}\nabla\ell^{i}(\boldsymbol{\Theta}_{t})\|_{S_{1}}^{2},
        \label{eq:moon_dual}
    \end{equation}
    where $\Delta_{m}$ denotes the probability simplex. Assume the optimum of the dual (\ref{eq:moon_dual}) is $z^{\ast}$ and the manipulated gradient $\boldsymbol{G}_t =\sum_{i=1}^{m} z_{i,t}^{\ast}\nabla\ell^{i}(\boldsymbol{\Theta}_t)$ has full rank. The exact optimal solution of (\ref{eq:moon_primal}) is given by
    $$
    \widehat{\boldsymbol{W}}_t = C\cdot\boldsymbol{U}_t \boldsymbol{V}_t^{\top},
    $$
    where $C = \|\sum_{i=1}^{m} z_{i,t}^{\ast}\nabla\ell^{i}(\boldsymbol{\Theta}_{t})\|_{S_{1}}$ is the nuclear norm. If $C=0$, then $\boldsymbol{\Theta}_t$ is Pareto stationary, and we set $\widehat{\boldsymbol{W}}_t=\boldsymbol{W}_t=\boldsymbol{0}$; otherwise, the following construction applies. $\boldsymbol{U}_t, \boldsymbol{V}_t$ are the left and right (compact) singular matrices of the manipulated gradient $\boldsymbol{G}_t$. 
    \label{prop:moon_dual}
\end{proposition}
Proposition \ref{prop:moon_dual} derives the dual of problem (\ref{eq:moon_primal}) and shows how to recover the exact primal solution from the dual optimum. For $C>0$, the orthonormalized direction used in Algorithm \ref{alg:moon} is $\boldsymbol{W}_t=\boldsymbol{U}_t\boldsymbol{V}_t^{\top}=\frac{1}{C}\widehat{\boldsymbol{W}}_t.$
Thus, $\boldsymbol{W}_t$ retains the direction of the exact minimax solution, while the update magnitude is controlled separately by the learning rate $\alpha$ in Step 6.

\paragraph{Remark: Comparison with Euclidean MOO methods.} 
For a fixed iterate $\boldsymbol{\Theta}_t$, define the weighted
aggregate gradient as $ \boldsymbol{G}_t(\boldsymbol{z}) = \sum_{i=1}^{m} z_i\nabla\ell^i(\boldsymbol{\Theta}_t), \boldsymbol{z}\in\Delta_m.$ Convexity and spectral--nuclear norm duality can imply the optimization meaning of its nuclear norm (Appendix \ref{app:remark2_proof}):
\begin{equation}
\label{eq:remark_convex_certificate}
\boldsymbol{z}^{\top}
    \boldsymbol{L}(\boldsymbol{\Theta}_t)
    -
    \min_{\boldsymbol{\Theta}\in\Omega}
    \boldsymbol{z}^{\top}\boldsymbol{L}(\boldsymbol{\Theta})
    \leq
    O\left(\left\|
    \boldsymbol{G}_t(\boldsymbol{z})
    \right\|_{S_1}\right).
\end{equation}
Thus,
$\|\boldsymbol{G}_t(\boldsymbol{z})\|_{S_1}$
provides a matrix-geometric certificate of the scalarized
suboptimality at $\boldsymbol{\Theta}_t$.

Let $\boldsymbol{z}_t^{\mathrm{MOON}}$ be an exact minimizer of
the nuclear-norm dual problem in
Equation~\eqref{eq:moon_dual}, and let
$\boldsymbol{z}_t^{\mathrm{E}}\in\Delta_m$ denote the weights
obtained by a Euclidean min-norm method such as MGDA. Since
$\boldsymbol{z}_t^{\mathrm{E}}$ is feasible for the MOON dual
problem, the optimality of
$\boldsymbol{z}_t^{\mathrm{MOON}}$ gives
\begin{equation}
\label{eq:remark_nuclear_dominance}
    \left\|
    \boldsymbol{G}_t
    \left(\boldsymbol{z}_t^{\mathrm{MOON}}\right)
    \right\|_{S_1}
    \leq
    \left\|
    \boldsymbol{G}_t
    \left(\boldsymbol{z}_t^{\mathrm{E}}\right)
    \right\|_{S_1}.
\end{equation}
Therefore, at the same iterate, the exact MOON weighting
provides a no-looser nuclear-norm-dependent certificate than
Euclidean min-norm weighting. Define the corresponding
certificate gap as
\begin{equation}
\label{eq:remark_weighting_gap}
    \operatorname{gap}_t
    :=
    \left\|
    \boldsymbol{G}_t
    \left(\boldsymbol{z}_t^{\mathrm{E}}\right)
    \right\|_{S_1}
    -
    \left\|
    \boldsymbol{G}_t
    \left(\boldsymbol{z}_t^{\mathrm{MOON}}\right)
    \right\|_{S_1}
    \geq 0.
\end{equation}
For a fixed sequence of iterates, the difference between the
corresponding certificate terms is
\begin{equation}
\label{eq:remark_certificate_gap}
    \Delta_{\mathrm{cert}}
    =
    \frac{1}{T}
    \sum_{t=1}^{T}
    \operatorname{gap}_t
    \geq 0,
\end{equation}
with a strict inequality whenever the Euclidean weighting is
not nuclear-norm optimal at least at one iterate\footnote{We provide a case study to show this advantage in Appendix~\ref{app:two_task_example}}.

\subsection{Practical Implementation}
\label{sec:practical_implementation}

Applying the exact MOON formulation to large-scale training
presents three practical challenges: fluctuations in the
aggregate gradient, the cost of exact polar-factor computation,
and the inner-loop optimization required to solve the dual
problem (\ref{eq:moon_dual}). To address these challenges, practical MOON incorporates
three components: (1) gradient momentum to stabilize the
aggregate gradient across iterations, (2) Newton--Schulz
iterations to efficiently approximate its polar factor, and
(3) a single-step online update to track the dual task weights. Together, these components yield the practical MOON procedure summarized in Algorithm~\ref{alg}. 

\paragraph{Gradient Momentum.}
At iteration $t$, MOON constructs the weighted aggregate
gradient
$
    \boldsymbol{G}_t
    =
    \sum_{i=1}^{m}
    z_{i,t}\nabla\ell^i(\boldsymbol{\Theta}_t).$
Because both the stochastic task gradients and the task weights
may change over time, directly constructing the matrix-aware
update from $\boldsymbol{G}_t$ can lead to substantial
iteration-to-iteration variation. We therefore maintain an
exponential moving average
\begin{equation}
\label{eq:gradient_momentum}
    \boldsymbol{M}_t
    =
    (1-\mu)\boldsymbol{M}_{t-1}
    +
    \mu\boldsymbol{G}_t,
    \qquad
    0<\mu\leq 1.
\end{equation}
The momentum matrix filters
short-term variations in the aggregate gradient and provides a
more stable matrix signal for constructing the subsequent
orthonormalized direction. 

\paragraph{Polar-Factor Approximation.}
Given the momentum matrix $\boldsymbol{M}_t$, MOON uses its
polar factor as the matrix-aware update direction:
\begin{equation}
\label{eq:momentum_polar}
    \boldsymbol{W}_t
    =
    \operatorname{Polar}(\boldsymbol{M}_t).
\end{equation}
Since
computing the exact polar factor through singular value
decomposition can be expensive, we use a finite number of
Newton--Schulz iterations to approximate
$\operatorname{Polar}(\boldsymbol{M}_t)$ similar to previous works
\cite{jordan2024muon,bernstein2024old,bjorck1971iterative,
kovarik1970some}. By convention, we define $\operatorname{Polar}(\boldsymbol{0})=\boldsymbol{0}$, so that $\boldsymbol{W}_t=\boldsymbol{0}$ whenever $\boldsymbol{M}_t=\boldsymbol{0}$.

\paragraph{Online Approximation of the Dual Weights.} Solving the nuclear-norm dual problem ~\eqref{eq:moon_dual} would introduce an expensive inner optimization loop. We instead use a single online update to track its solution, following the approximation strategy of FAMO \citep{liu2024famo}. We use the following momentum-based, scale-normalized surrogate direction for tracking the dual weights.:
\begin{equation} \label{eq:practical_dual_direction} \boldsymbol{\delta}_t = \left[ \left\langle \nabla\ell^1(\boldsymbol{\Theta}_t), \boldsymbol{W}_t \right\rangle, \ldots, \left\langle \nabla\ell^m(\boldsymbol{\Theta}_t), \boldsymbol{W}_t \right\rangle \right]^\top. 
\end{equation} 
To maintain the simplex constraint, we parameterize the task weights using logits and update 
\begin{equation} \label{eq:logit_update} 
\boldsymbol{\xi}_{t+1} = \boldsymbol{\xi}_t - \beta \left( \boldsymbol{\delta}_t + \gamma\boldsymbol{\xi}_t \right), \boldsymbol{z}_{t+1} = \operatorname{Softmax} \left( \boldsymbol{\xi}_{t+1} \right), 
\end{equation} 
where $\beta$ is the weight update stepsize, and $\gamma$ regularizes the logits. It provides an efficient momentum-based approximation to the exact dual solution while ensuring $\boldsymbol{z}_{t+1}\in\Delta_m$. 

\begin{algorithm}[t]
   \caption{Multi-Objective OrthoNormalized updates}
   \label{alg}
\begin{algorithmic}[1]
   \STATE \label{moon:1} {\bfseries Input:} Initial parameter $\boldsymbol{\Theta}_1, \boldsymbol{M}_0 \leftarrow 0$, loss functions $\{\ell^{i}\}_{i=1}^{m}$, learning rate $\alpha$, $\beta$, momentum $0<\mu\leq1$, weight decay of logits update $\gamma$, $\boldsymbol{\xi}_1\leftarrow \boldsymbol{0}$.
   \FOR{$t=1$ {\bfseries to} $T$} \label{moon:2}
   \STATE \label{moon:3} Normalize weight $\boldsymbol{z}_t\leftarrow \text{Softmax}(\boldsymbol{\xi}_t)$
   \STATE \label{moon:4} Compute composite gradient: $$\boldsymbol{G}_t = \sum_{i=1}^{m}z_{i,t}\nabla \ell^{i}(\boldsymbol{\Theta}_t)$$
   \STATE \label{moon:momemtum} Compute momentum: $\boldsymbol{M}_t = (1-\mu)\boldsymbol{M}_{t-1}+\mu \boldsymbol{G}_t$
   \STATE \label{moon:5} Perform orthonormalization on momentum by: $$\boldsymbol{W}_t = \boldsymbol{U}_t \boldsymbol{V}_t^{\top}\text{ (approximated by Newton-Schulz)}$$
   \STATE \label{moon:6} Update model parameter $\boldsymbol{\Theta}_{t+1} = \boldsymbol{\Theta}_t - \alpha\boldsymbol{W}_t$
   \STATE \label{moon:7}Update weight logits by: $$\boldsymbol{\xi}_{t+1} = \boldsymbol{\xi}_{t} - \beta(\boldsymbol{\delta}_{t} + \gamma\boldsymbol{\xi}_{t})$$ where $\boldsymbol{\delta}_{t} = [\langle \boldsymbol{W}_t, \nabla\ell^{1}(\boldsymbol{\Theta}_t)\rangle,\dots,\langle \boldsymbol{W}_t, \nabla\ell^{m}(\boldsymbol{\Theta}_t)\rangle]^\top$
   \ENDFOR
\end{algorithmic}
\label{alg:moon}
\end{algorithm}

\subsection{Convergence Analysis}
We analyze the convergence of MOON for smooth non-convex objectives. For any $\boldsymbol{\Theta}$, define the nuclear-norm Pareto-stationarity measure 
\begin{equation} \label{eq:pareto_stationarity_measure} 
\mathcal{G}(\boldsymbol{\Theta}) := \min_{\boldsymbol{z}\in\Delta_m} \left\| \sum_{i=1}^{m} z_i\nabla\ell^i(\boldsymbol{\Theta}) \right\|_{S_1}. 
\end{equation} 
In the unconstrained setting, $\mathcal{G}(\boldsymbol{\Theta})=0$ if and only if $\boldsymbol{\Theta}$ is Pareto stationary. We therefore study the averaged measure $\frac{1}{T}\sum_{t=1}^{T}\mathcal{G}(\boldsymbol{\Theta}_t)$. We first consider the deterministic setting.

\begin{theorem}[Deterministic Convergence]
\label{thm:deterministic_convergence}
Suppose that each objective
$\ell^i:\mathbb{R}^{p\times q}\rightarrow\mathbb{R}$ is
$L$-smooth with respect to the spectral norm. Assume that there exist constants $H>0$ and $B>0$,
independent of both $T$ and $t$, such that $\max_{t\in [T+1]}\max_{i\in[m]}\|\nabla\ell^i(\boldsymbol{\Theta}_t)\|_{S_1}\le H$ and $\max_{t\in [T+1]}\max_{i\in[m]}|\ell^i(\boldsymbol{\Theta}_t)|\le B$. Choose $\alpha=\sqrt{\frac{B}{LT}}$, $\beta=\frac{1}{mHT}$ and $0<\gamma<mHT$. 
Then, for every fixed $\mu\in(0,1]$, the iterates generated by
Algorithm~\ref{alg:moon} satisfy
\[
\frac{1}{T}
\sum_{t=1}^{T}
\min_{\boldsymbol{z}_t^\ast\in\Delta_m}
\left\|
\nabla L(\boldsymbol{\Theta}_t)\boldsymbol{z}_t^\ast
\right\|_{S_1}
\le
O(T^{-1/2}). 
\]
\end{theorem}
The detailed proof is provided in Appendix \ref{app:deterministic proof}. Consequently, we demonstrate that our algorithm achieves a convergence rate of $O(1/\sqrt{T})$, showing that MOON converges to Pareto stationarity at a standard sublinear rate for smooth non-convex optimization. We next consider unbiased stochastic task gradients. 

\begin{theorem}[Stochastic Convergence]
    Under the same assumption as Theorem~\ref{thm:deterministic_convergence}. We assume that the stochastic gradient estimator is unbiased and has variance bounded by $\sigma$. Choosing $\alpha = \sqrt{\frac{\mu B}{TL}}$, $\beta = \frac{1}{(H+\sqrt{m}\rho\sigma)mT}$, $0<\gamma<(H+\sqrt{m}\rho\sigma)mT$ and $\mu = \min\{ \frac{\sqrt{LB}}{\sigma\rho\sqrt{T}},1\}$, where $\rho := \sqrt{\min \{p,q\}}$, then the iterations of Algorithm \ref{alg:moon} satisfy 
    \begin{equation}
            \frac{1}{T}\sum_{t=1}^T \mathbb{E}\left[\min_{\boldsymbol{z}_t^\ast \in \Delta_m}\|\nabla L(\boldsymbol{\Theta_t})\boldsymbol{z}_t^\ast \|_{S_1}\right] \leq \mathcal{O}(1/T^{1/4}). 
    \end{equation}
\end{theorem}
The detailed proof is provided in Appendix \ref{section:algo_stochastic_case}. The expected Pareto-stationarity measure of MOON satisfies $\mathcal{O}(T^{-1/4})$ bound. This slower rate reflects the additional error introduced by gradient noise and its interaction with momentum, and is also standard under stochastic gradients.

\section{Related Work}

\paragraph{Multi-objective Optimization.}
Multi-objective optimization (MOO) optimizes multiple tasks simultaneously, with the central challenge of resolving conflicts among task gradients \citep{sener2018multi, yu2020gradient, liu2021conflict, navon2022multi, liu2024famo}. Methods most relevant to our work follow two main lines. Loss balancing approaches \citep{kendall2018multi, liu2019end, lin2021closer} adjust objective weights using loss statistics, optimizing an explicit weighted objective whose update is a linear combination of task-specific updates. These methods are often motivated by empirical heuristics and have limited theoretical support. Gradient manipulation methods \citep{sener2018multi, yu2020gradient, liu2021conflict, liu2021towards, navon2022multi, fernando2023mitigating, liu2024famo} instead modify task gradients to obtain a common update direction that balances optimization across tasks. They generally provide stronger theoretical guarantees and often achieve better performance than loss balancing. However, existing methods typically flatten all parameters into a single vector, preventing them from attaining the steepest descent direction in the native non-Euclidean geometry of structured parameters. Our method preserves the matrix structure of the parameters and performs gradient manipulation directly in this geometry, improving both training efficiency and performance.

\paragraph{Optimization for matrix-valued parameters.}
Modern architectures such as Transformers \citep{vaswani2017attention} are dominated by matrix-valued parameters, motivating optimizers that exploit this structure. K-FAC approximates layerwise Fisher blocks using Kronecker factors \citep{martens2015optimizing}; Adafactor factorizes row- and column-wise second moments to reduce memory usage \citep{shazeer2018adafactor}; and Shampoo uses tensor-structured preconditioners to improve conditioning and accelerate convergence \citep{gupta2018shampoo}. Muon instead approximately orthogonalizes momentum matrices for hidden-layer weights, producing efficient matrix-aware updates \citep{jordan2024muon}. Its theoretical properties have been studied from complementary perspectives: Li and Hong establish convergence guarantees \citep{li2025note}, while Kovalev interprets gradient orthogonalization as non-Euclidean trust-region optimization \citep{kovalev2025understanding}.

Subsequent work improves the efficiency, stability, and scalability of Muon. MuonBP reduces distributed communication through block-periodic orthogonalization \citep{khaled2025muonbp}, while Polar Express develops GPU-friendly matrix-sign iterations for efficient polar-factor computation \citep{amsel2025polar}. MuonClip improves large-model training stability by controlling attention-logit growth through QK clipping \citep{team2025kimi}, and large-scale studies further demonstrate Muon's practical efficiency and scalability in language-model pretraining \citep{shah2025practical, liu2025muon}. Despite these advances, such methods optimize a single aggregate objective and do not directly address conflicts among task objectives. MOON extends matrix-aware optimization to the multi-objective setting by combining task gradients according to the geometry of matrix-valued parameters and applying orthonormalized updates. This preserves parameter structure while resolving task conflicts, improving training efficiency and final performance in our experiments.

\section{Experiments}

In this section, we conduct experiments with MOON on several benchmarks to evaluate the performance and convergence. Our experiments cover simulation, multi-task scene understanding, classification, and regression. 

\paragraph{Experimental Setup.} We consider 5 benchmarks widely used in multi-task learning research: MultiMNIST (2 tasks) \citep{sabour2017dynamic}, NYU-v2 (3 tasks) \citep{silberman2012indoor}, CityScapes (2 tasks) \citep{Cordts2016Cityscapes}, QM9 (11 tasks) \citep{blum2009970} and CelebA (40 tasks) \citep{liu2015deep}. We compare the proposed method with 12 multi-objective optimization baselines: Single-task learning (STL), Linear Scalarization (LS), a scale-invariant baseline (SI) that applies equal weighting to the log losses, Random Loss Weighting (RLW) \citep{lin2021closer}, Dynamic Weight Average (DWA) \citep{liu2019end}, Uncertainty Weighting (UW) \citep{kendall2018multi}, MGDA \citep{sener2018multi}, PCGrad \citep{yu2020gradient}, CAGrad \citep{liu2021conflict}, IMTL-G \citep{liu2021towards}, Nash-MTL \citep{navon2022multi}, FAMO \citep{liu2024famo}. We discuss each experiment in the following subsections.\footnote{For more details about these experiments, see Appendix \ref{section:experimental_details}.}

\begin{table}[!t]
\begingroup
\setcounter{table}{1}
\renewcommand{\theHtable}{nyuv2}
\centering
\begin{small}
\begin{sc}
\setlength{\tabcolsep}{4pt}

\resizebox{\textwidth}{!}{%
\begin{tabular}{lcc cc cc ccc c}
\toprule
\multirow{3}{*}{\textbf{Method}} &
\multicolumn{2}{c}{\textbf{Segmentation}} &
\multicolumn{2}{c}{\textbf{Depth}} &
\multicolumn{5}{c}{\textbf{Surface Normal}} &
\multirow{3}{*}{$\boldsymbol{\Delta m}\%\,\downarrow$} \\
\cmidrule(lr){2-3}\cmidrule(lr){4-5}\cmidrule(lr){6-10}
& mIoU$\uparrow$ & Pix Acc$\uparrow$ &
Abs Err$\downarrow$ & Rel Err$\downarrow$ &
\multicolumn{2}{c}{Angle Dist$\downarrow$} &
\multicolumn{3}{c}{Within $t^\circ\uparrow$} & \\
\cmidrule(lr){6-7}\cmidrule(lr){8-10}
& & & & &
Mean & Median & 11.25 & 22.5 & 30 & \\
\midrule

STL      & 38.30 & 63.76 & 0.6754 & 0.2780 & 25.01 & 19.21 & 30.14 & 57.20 & 69.15 & \\
\cmidrule(lr){1-11}
LS       & 39.29 & 65.33 & 0.5493 & 0.2263 & 28.15 & 23.96 & 22.09 & 47.50 & 61.08 & 5.59 \\
SI       & 38.45 & 64.27 & 0.5354 & 0.2201 & 27.60 & 23.37 & 22.53 & 48.57 & 62.32 & 4.39 \\
RLW      & 37.17 & 63.77 & 0.5759 & 0.2410 & 28.27 & 24.18 & 22.26 & 47.05 & 60.62 & 7.78 \\
DWA      & 39.11 & 65.31 & 0.5510 & 0.2285 & 27.61 & 23.18 & 24.17 & 50.18 & 62.39 & 3.57 \\
UW       & 36.87 & 63.17 & 0.5446 & 0.2260 & 27.04 & 22.61 & 23.54 & 49.05 & 63.65 & 4.05 \\
MGDA     & 28.40 & 58.24 & 0.5979 & 0.2335 & \textbf{25.00} & 19.49 & 29.56 & 56.84 & 69.11 & 1.23 \\
PCGRAD   & 39.22 & 65.98 & 0.5373 & 0.2273 & 27.27 & 23.08 & 23.17 & 49.23 & 62.98 & 3.40 \\
GradDrop & 39.39 & 65.12 & 0.5455 & 0.2279 & 27.48 & 22.96 & 23.38 & 49.44 & 62.87 & 3.58 \\
CAGrad   & 39.00 & 65.77 & 0.5561 & 0.2135 & 26.30 & 21.53 & 25.86 & 52.47 & 65.65 & -0.12 \\
IMTL-G   & 39.35 & 65.60 & 0.5426 & 0.2256 & 26.02 & 21.19 & 26.20 & 53.13 & 66.24 & -0.76 \\
NashMTL  & \textbf{40.13} & 65.93 & 0.5261 & 0.2171 & 25.26 & 20.08 & 28.40 & 55.47 & 68.15 & -4.04 \\
FAMO     & 36.41 & 63.20 & 0.5406 & 0.2182 & 25.01 & \textbf{19.14} & \textbf{30.22} & \textbf{57.45} & \textbf{69.37} & -4.11 \\
\cmidrule(lr){1-11}
MOON (Ours)     & 39.41 & \textbf{67.03} & \textbf{0.4891} & \textbf{0.2091} & 25.61 & 20.27 & 28.89 & 54.87 & 67.32 & \textbf{-4.63} \\
\bottomrule
\end{tabular}
}

\end{sc}
\end{small}
\caption{Results on NYU-v2 (3 tasks) dataset. Each experiment is repeated over 3 random seeds and the mean is reported. The best average result is marked in bold. The task specific metrics and the average performance drop $\boldsymbol{\Delta m}\%$ are reported.}

\label{tab:nyuv2}
\endgroup
\setcounter{table}{0}
\end{table}

\subsection{A Toy Case}
\label{subsec:toycase}

To better understand why we cannot directly apply structure-aware optimizers such as Muon \citep{jordan2024muon} to existing MOO methods, we design a toy experiment to provide an intuitive illustration.

In this experiment, we construct a 2-task optimization problem based on a multi-objective least-squares formulation. Given two 6-dimensional vectors $\boldsymbol{x}$ and $\boldsymbol{y}$, we define two objectives $L_{1} = \frac{1}{6}\|\boldsymbol{\Theta}\boldsymbol{x}-\boldsymbol{y}\|_{2}^{2}$ and $L_{2} = \frac{1}{6}\|\boldsymbol{\Theta}\boldsymbol{x}+\boldsymbol{y}\|_{2}^{2}$ over a matrix parameter $\boldsymbol{\Theta}\in\mathbb{R}^{6\times 6}$. We optimize $\boldsymbol{\Theta}$ using three methods: (i) MGDA \citep{sener2018multi}, (ii) MGDA with Muon optimizer \citep{jordan2024muon} applied directly, and (iii) MOON. Then we plot the average of the two losses during the 1500-step optimization in Figure \ref{fig:toy1_avgloss}.

Figure \ref{fig:toy1_avgloss} shows the average loss of the two tasks over training steps in this toy experiment. MOON converges faster and reaches a lower final average loss, indicating a more balanced optimization on matrix-valued parameters. In contrast, directly applying Muon to MGDA fails to achieve the same improvement.

\begin{figure}[t]
\begin{center}
\centering
\includegraphics[width=0.65\columnwidth]{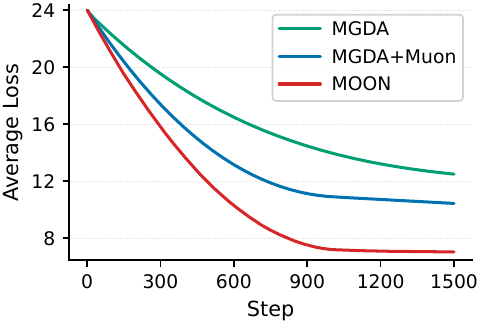}
\caption{Average loss curve on the toy case.}
\Description{Line chart comparing MGDA, MGDA combined with Muon, and MOON over 1500 optimization steps. MOON decreases the average loss fastest and reaches the lowest final value.}
\label{fig:toy1_avgloss}
\end{center}
\end{figure}

\subsection{Multi-task Classification}

We evaluate our method on MultiMNIST \citep{sabour2017dynamic} with Vision Transformer (ViT) \citep{dosovitskiy2020image} backbone to assess convergence and performance under matrix-structured parameters. The ViT contains matrix-valued weights in the Transformer encoder (attention projections and MLP linear layers). ViT is composed of $76.88\%$ matrix parameters. We treat classifying the left and right digits as two tasks, reporting per-task cross-entropy training loss for convergence and test accuracy for performance. An additional classification experiment is conducted on CelebA, adopting a CNN backbone with $99.46\%$ learnable convolutional weights treated as matrix parameters. In our experiments, we follow the treatment of convolutional weights in Muon \citep{jordan2024muon}, converting 4D convolutional parameters into matrices by flattening the last three dimensions. It evaluates training efficiency and performance with a large number of objectives (40 tasks).

In the MultiMNIST experiment, we observe from Figure \ref{fig:multimnist_training_loss} that our method reduces the training loss for both tasks more rapidly in the early stage and maintains a consistent advantage throughout training. Table \ref{tab:multimnist} reports the test accuracies for the two classification tasks and their average. MOON achieves the best performance on both tasks and the highest average accuracy among all compared MOO methods. The right-hand side of Table \ref{tab:cityscapes} shows that MOON achieves performance comparable to state-of-the-art methods on the 40-task CelebA benchmark, further supporting the results.

\begin{figure}[!t]
\centering
\begin{minipage}[t]{0.49\columnwidth}
\centering
\includegraphics[width=\linewidth]{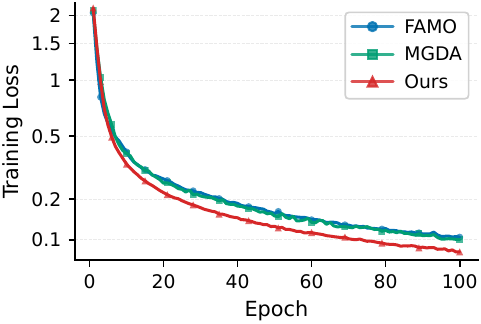}\par
\small (a) ViT left task
\end{minipage}\hfill
\begin{minipage}[t]{0.49\columnwidth}
\centering
\includegraphics[width=\linewidth]{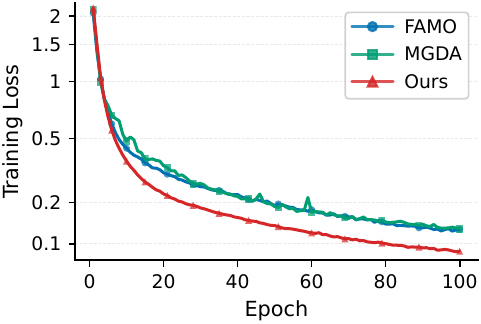}\par
\small (b) ViT right task
\end{minipage}
\caption{Convergence of per-task training cross-entropy losses vs. training epochs on MultiMNIST.}

\label{fig:multimnist_training_loss}
\end{figure}

\begin{table}[!t]
\centering
\begin{small}
\begin{sc}

\centering
\begin{tabular}{lccc}
\toprule
Method & Left Acc & Right Acc & Average Acc \\
\midrule
MGDA      & 95.58 & 94.10 & 94.84 \\
PCGrad    & 94.26 & 93.60 & 93.93 \\
CAGrad    & 95.65 & 94.72 & 95.19 \\
IMTL-G    & 94.57 & 94.43 & 94.50 \\
Nash-MTL  & 95.33 & 94.38 & 94.86 \\
FAMO      & 95.89 & 94.88 & 95.39 \\
MOON (Ours) & \textbf{95.99} & \textbf{95.31} & \textbf{95.65} \\
\bottomrule
\end{tabular}
\par\smallskip

\end{sc}
\end{small}
\caption{Test accuracy (\%) on MultiMNIST (2 tasks). Per-task accuracies (left/right) and their average are reported.}
\label{tab:multimnist}
\vspace{-3mm}
\end{table}
\setcounter{table}{2}

\begin{figure}[!t]
\centering
\begin{minipage}[t]{0.66\textwidth}
\vspace{0pt}
\centering
\begin{small}
\begin{sc}
\setlength{\tabcolsep}{3.2pt}
\resizebox{\linewidth}{!}{
\begin{tabular}{lcccccc}
\toprule
& \multicolumn{5}{c}{\textbf{CityScapes}} & \textbf{CelebA} \\
\cmidrule(lr){2-6}\cmidrule(lr){7-7}
\multirow{2}{*}{\textbf{Method}} & \multicolumn{2}{c}{Segmentation} & \multicolumn{2}{c}{Depth} & \multirow{2}{*}{$\boldsymbol{\Delta m}\%\,\downarrow$} & \multirow{2}{*}{$\boldsymbol{\Delta m}\%\,\downarrow$} \\
\cmidrule(lr){2-3}\cmidrule(lr){4-5}
& mIoU$\uparrow$ & Pix Acc$\uparrow$ & Abs Err$\downarrow$ & Rel Err$\downarrow$ & & \\
\midrule
STL & 74.01 & 93.16 & 0.0125 & 27.77 & & \\
\cmidrule{1-7}
LS & 70.95 & 91.73 & 0.0161 & 33.83 & 14.11 & 6.28 \\
RLW & 74.57 & 93.41 & 0.0158 & 47.79 & 24.38 & 5.22 \\
DWA & 75.24 & 93.52 & 0.0160 & 44.37 & 21.45 & 6.95 \\
UW & 72.02 & 92.85 & 0.0140 & 30.13 & 5.89 & 5.78 \\
MGDA & 69.60 & 91.87 & 0.0134 & 36.05 & 11.02 & 10.93 \\
PCGrad & 75.13 & 93.48 & 0.0154 & 42.07 & 18.29 & 6.65 \\
GradDrop & 75.27 & 93.53 & 0.0157 & 47.54 & 23.73 & 7.80 \\
CAGrad & 75.16 & 93.48 & 0.0141 & 37.60 & 11.64 & 6.20 \\
IMTL-G & 75.33 & 93.49 & 0.0135 & 38.41 & 11.10 & 4.67 \\
NashMTL & 75.41 & 93.66 & 0.0129 & 35.02 & 6.82 & 4.97 \\
FAMO & 72.48 & 92.56 & 0.0148 & \textbf{29.52} & 6.93 & 4.72 \\
\cmidrule{1-7}
MOON (Ours) & \textbf{78.61} & \textbf{94.36} & \textbf{0.0126} & 31.41 & \textbf{1.54} & \textbf{4.65} \\
\bottomrule
\end{tabular}
}
\end{sc}
\end{small}
\captionof{table}{Results on CityScapes (2 tasks) and CelebA (40 tasks) dataset. Each experiment is repeated over 3 random seeds and the mean is reported. The best average result is marked in bold. Task-specific metrics and the average performance drop $\boldsymbol{\Delta m}\%$ are reported.}
\label{tab:cityscapes}
\end{minipage}
\hfill
\begin{minipage}[t]{0.31\textwidth}
\vspace{0pt}
\centering
\includegraphics[width=0.90\linewidth]{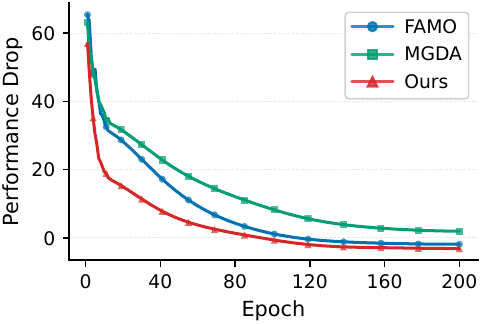}\par
\makebox[\linewidth][c]{\small (a) NYU-v2}\par
\vspace{1.5mm}
\includegraphics[width=0.90\linewidth]{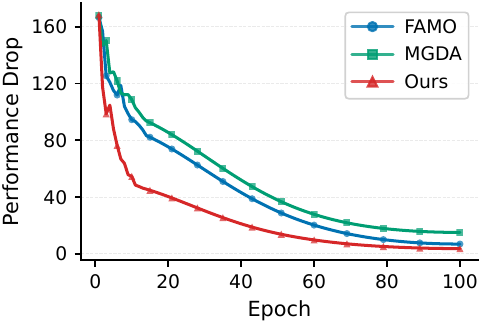}\par
\makebox[\linewidth][c]{\small (b) CityScapes}\par
\captionof{figure}{Performance drop $\boldsymbol{\Delta m}\%$ relative to STL vs. training epochs on NYU-v2 and CityScapes benchmarks.}
\label{fig:nyu_city_deltap}
\Description{Two line charts compare performance drop during training on NYU-v2 and CityScapes. MOON reduces the performance drop more rapidly and reaches lower values than MGDA and FAMO.}
\end{minipage}
\end{figure}

\begin{table}[!t]
\centering
\begin{small}
\begin{sc}
\setlength{\tabcolsep}{4pt}

\resizebox{\textwidth}{!}{%
\begin{tabular}{lcccccccccccc}
\toprule
\multirow{2}{*}{Method} &
$\mu$ & $\alpha$ & $\epsilon_{\mathrm{HOMO}}$ & $\epsilon_{\mathrm{LUMO}}$ &
$\langle R^2\rangle$ & ZPVE & $U_0$ & $U$ & $H$ & $G$ & $c_v$ &
$\boldsymbol{\Delta m}\%\downarrow$ \\
\cmidrule(lr){2-12}
& \multicolumn{11}{c}{MAE$\downarrow$} & \\
\midrule
STL     & 0.07 & 0.18 & 60.6  & 53.9  & 0.50 & 4.53  & 58.8  & 64.2  & 63.8  & 66.2  & 0.07 & \\
\midrule
LS      & 0.11 & 0.33 & 73.6  & 89.7  & 5.20 & 14.06 & 143.4 & 144.2 & 144.6 & 140.3 & 0.13 & 177.6 \\
SI      & 0.31 & 0.35 & 149.8 & 135.7 & \textbf{1.00} & \textbf{4.51}  & 55.3  & 55.8  & 55.8  & 55.3  & 0.11 & 77.8  \\
RLW     & 0.11 & 0.34 & 76.9  & 92.8  & 5.87 & 15.47 & 156.3 & 157.1 & 157.6 & 153.0 & 0.14 & 203.8 \\
DWA     & 0.11 & 0.33 & 74.1  & 90.6  & 5.09 & 13.99 & 142.3 & 143.0 & 143.4 & 139.3 & 0.13 & 175.3 \\
UW      & 0.39 & 0.43 & 166.2 & 155.8 & 1.07 & 4.99  & 66.4  & 66.8  & 66.8  & 66.2  & 0.12 & 108.0 \\
MGDA    & 0.22 & 0.37 & 126.8 & 104.6 & 3.23 & 5.69  & 88.4  & 89.4  & 89.3  & 88.0  & 0.12 & 120.5 \\
PCGrad  & 0.11 & 0.29 & 75.9  & 88.3  & 3.94 & 9.15  & 116.4 & 116.8 & 117.2 & 114.5 & 0.11 & 125.7 \\
CAGrad  & 0.12 & 0.32 & 83.5  & 94.8  & 3.22 & 6.93  & 114.0 & 114.3 & 114.5 & 112.3 & 0.12 & 112.8 \\
IMTL-G  & 0.14 & 0.29 & 98.3  & 93.9  & 1.75 & 5.70  & 101.4 & 102.4 & 102.0 & 100.1 & 0.10 & 77.2  \\
NashMTL & 0.10 & 0.25 & 82.9  & 81.9  & 2.43 & 5.38  & 74.5  & 75.0  & 75.1  & 74.2  & 0.09 & 62.0  \\
FAMO    & 0.16 & 0.28 & 85.0  & 87.0  & 1.81 & 5.13  & 68.2 & 68.5  & 68.7  & 67.8  & 0.09 & 57.3  \\
\midrule
MOON (Ours)    & \textbf{0.07} & \textbf{0.24} & \textbf{52.9}  & \textbf{71.0}  & 3.35 & 5.67  & \textbf{46.4} & \textbf{46.7}  & \textbf{46.9}  & \textbf{46.3}  & \textbf{0.08} & \textbf{49.9}  \\

\bottomrule
\end{tabular}
}

\end{sc}
\end{small}
\caption{Results on QM9 dataset (11 tasks). Each experiment is repeated over 3 random seeds and the mean is reported. The best average result is marked in bold. The task specific metrics and the average performance drop $\boldsymbol{\Delta m}\%$ are reported.}

\label{tab:qm9}
\end{table}

\subsection{Multi-task Scene Understanding}

We apply MOON to multi-task scene understanding and evaluate it on the NYU-v2 \citep{silberman2012indoor} and CityScapes \citep{Cordts2016Cityscapes} benchmarks. We adopt the SegNet-MTAN architecture \cite{badrinarayanan2017segnet, liu2019end}, a standard backbone for multi-task dense prediction. In this model, matrix-structured parameters or convolutional filters arise in the shared SegNet encoder-decoder as well as the MTAN attention, which account for $99.57\%$ of all parameters. We measure multi-objective optimization performance using the average performance drop $\boldsymbol{\Delta m}\%$ relative to the single-task learning (STL) baseline \citep{navon2022multi}. For each method $m$, the performance drop is calculated by $\boldsymbol{\Delta m}\% = \frac{1}{K} \sum_{i=1}^{K} (-1)^{\delta_{i}} (M_{i}^{m}-M_{i}^{STL})\cdot 100/M_{i}^{STL}$, where $M_{i}^{m}$ denotes the $i$-th metric achieved by method $m$, and $\delta_{i}=0$ if higher values are better for metric $i$; otherwise $\delta_{i}=1$. Table \ref{tab:nyuv2} and Table \ref{tab:cityscapes} summarize the results on NYU-v2 and CityScapes, respectively. Figure \ref{fig:nyu_city_deltap} demonstrates how the test average performance drop $\boldsymbol{\Delta m}\%$ decreases during training.

From Table \ref{tab:nyuv2} and Table \ref{tab:cityscapes}, we observe that our method achieves state-of-the-art performance on both benchmarks. In particular, on CityScapes, it attains the best results on most task metrics. Figure \ref{fig:nyu_city_deltap} shows that our method exhibits faster descent on the performance drop $\boldsymbol{\Delta m}\%$ than FAMO. This demonstrates that MOON not only converges faster but also achieves better results than traditional MOO methods, supporting our theoretical claims with practical evidence.

\subsection{Multi-task Regression}

We further evaluate MOON on multi-task regression using the QM9 molecular property prediction benchmark \citep{blum2009970}. Following standard practice, we treat predicting each of the 11 quantum-chemical properties as a separate task. We adopt a graph-level multi-task regression network, including matrix-valued parameters like edge-conditioned linear transformations in NNConv and the Set2Set internal LSTM weights \citep{gilmer2017neural}, which account for $98.91\%$ of all parameters. Details of the parameterization are provided in Appendix \ref{section:experimental_details}. In Table \ref{tab:qm9} we report the mean absolute error (MAE) for each property. To summarize multi-objective performance, we also report the average performance drop $\boldsymbol{\Delta m}\%$ relative to the STL baseline.

Table \ref{tab:qm9} shows that MOON achieves the best overall performance, attaining the lowest performance drop $\boldsymbol{\Delta m}\%$ and the lowest MAE on most individual tasks. The reason is that gradient manipulation conducted in non-Euclidean spaces is more suitable for addressing matrix gradient conflicts.

\section{Conclusion}

This work addresses the geometric mismatch between conventional Euclidean MOO methods and the matrix-valued parameters prevalent in modern architectures such as Transformers. We proposed MOON (Multi-Objective OrthoNormalized Updates), which performs gradient manipulation under spectral--nuclear norm geometry and constructs matrix-aware orthonormalized updates. At any fixed iterate, the exact MOON weighting provides a nuclear-norm-based stationarity certificate no larger than that obtained by Euclidean min-norm weighting, thereby establishing its optimality for the matrix-geometric dual subproblem. For smooth non-convex objectives, we established convergence of the averaged Pareto-stationarity measure at rates of $\mathcal{O}(T^{-1/2})$ in the deterministic setting and $\mathcal{O}(T^{-1/4})$ in the stochastic setting. Experiments demonstrate that MOON improves optimization efficiency while achieving competitive or improved final performance.

\section{Acknowledgment}
This work was supported in part by the Beijing Major Science and Technology Project under Contract no. Z251100008125031, National Science Foundation of China (62401327), and the MOE Project of
Key Research Institute of Humanities and Social Sciences (22JJD110001).  This work  was supported  by Beijing Academy of Artificial Intelligence (BAAI).

\bibliographystyle{plainnat}
\bibliography{references}

\clearpage
\appendix
\setcounter{secnumdepth}{2}

\begin{center}
{\LARGE\bfseries Appendix}
\end{center}

\allowdisplaybreaks[4]
\setlength{\emergencystretch}{2em}

\section{Missing Proofs}
\label{section:missiong_proofs}  

Throughout this appendix, we consider the idealized setting in which
$\boldsymbol W_t$ is the exact polar factor of the momentum matrix
$\boldsymbol M_t$, i.e.,
\[
    \boldsymbol W_t
    =
    \operatorname{Polar}(\boldsymbol M_t).
\]
The effect of replacing the exact polar factor with a finite-step
Newton--Schulz approximation is analyzed separately in
Appendix~\ref{app:NS_error}.

\subsection{Description for Assumption}

In this paper, we have the following assumptions. 

\begin{assumption}\label{ass:L-smoothness}
For each $i \in [m]$, the objective function $\ell^i:\mathbb{R}^{p \times q}  \rightarrow \mathbb{R}$ is differentiable and $L$-smooth with respect to the spectral norm $\| \cdot \|_{S_\infty}$, i.e., there exist a bounded constant $L>0$ such that for any $\boldsymbol{\Theta},\boldsymbol{\Theta}' \in \mathbb{R}^{p\times q}$, $\| \nabla \ell^i (\boldsymbol{\Theta}) - \nabla \ell^i (\boldsymbol{\Theta}') \|_{S_1} \leq L \| \boldsymbol{\Theta} - \boldsymbol{\Theta}' \|_{S_\infty}$.
\end{assumption}

\begin{assumption}\label{ass:bounded_gradient}
There exist a constant $H > 0$ such that for any $i\in [m]$ and $t=1,\dots,T+1$, $\| \nabla \ell^i(\boldsymbol{\Theta}_t) \|_{S_1} \leq H$.
\end{assumption}

\begin{assumption}\label{ass:bounded_function}
There exist a constant $B > 0$ such that for any $i\in [m]$ and $t=1,\dots,T+1$, $| \ell^i(\boldsymbol{\Theta}_t) | \leq B$.
\end{assumption}

\subsection{Proof of Remark 2}\label{app:remark2_proof}

Suppose that every $\ell^i:\Omega \subseteq \mathbb{R}^{p\times q}\rightarrow\mathbb{R}$ is convex, where $\Omega$ is a nonempty compact
convex set, and that the spectral-norm diameter of $\Omega$ is uniformly bounded 
\[
    D
    :=
    \sup_{\boldsymbol{\Theta},\boldsymbol{\Theta}'\in \Omega}
    \left\|
    \boldsymbol{\Theta}-\boldsymbol{\Theta}'
    \right\|_{S_\infty}.
\]
For any $\boldsymbol{z}\in\Delta_m$, let
\[
    \boldsymbol{\Theta}_t^\ast(\boldsymbol{z})
    \in
    \argmin_{\boldsymbol{\Theta}\in\Omega}
    \boldsymbol{z}^{\top}\boldsymbol{L}(\boldsymbol{\Theta}).
\]
Convexity and spectral--nuclear norm duality imply
\begin{equation}
\begin{aligned}
    &\boldsymbol{z}^{\top}
    \boldsymbol{L}(\boldsymbol{\Theta}_t)
    -
    \min_{\boldsymbol{\Theta}\in\Omega}
    \boldsymbol{z}^{\top}\boldsymbol{L}(\boldsymbol{\Theta})
    \\
    &\quad\leq
    \left\langle
    \boldsymbol{G}_t(\boldsymbol{z}),
    \boldsymbol{\Theta}_t
    -
    \boldsymbol{\Theta}_t^\ast(\boldsymbol{z})
    \right\rangle
    \\
    &\quad\leq
    D
    \left\|
    \boldsymbol{G}_t(\boldsymbol{z})
    \right\|_{S_1}.
\end{aligned}
\end{equation}

\subsection{Convergence Analysis for Deterministic Gradient}\label{app:deterministic proof}

\begin{theorem}[Deterministic Convergence]
\label{thm:deterministic_convergence_app}
Suppose that each objective
$\ell^i:\mathbb{R}^{p\times q}\rightarrow\mathbb{R}$ is
$L$-smooth with respect to the spectral norm. Assume that there exist constants $H>0$ and $B>0$,
independent of both $T$ and $t$, such that $\max_{t\in [T+1]}\max_{i\in[m]}\|\nabla\ell^i(\boldsymbol{\Theta}_t)\|_{S_1}\le H$ and $\max_{t\in [T+1]}\max_{i\in[m]}|\ell^i(\boldsymbol{\Theta}_t)|\le B$. Choose $\alpha=\sqrt{\frac{B}{LT}}$, $\beta=\frac{1}{mHT}$ and $0<\gamma<mHT$. 
Then, for every fixed $\mu\in(0,1]$, the iterates generated by
Algorithm~\ref{alg:moon} satisfy
\[
\frac{1}{T}
\sum_{t=1}^{T}
\min_{\boldsymbol{z}_t^\ast\in\Delta_m}
\left\|
\nabla L(\boldsymbol{\Theta}_t)\boldsymbol{z}_t^\ast
\right\|_{S_1}
\le
O(T^{-1/2}). 
\]
\end{theorem}

Before the final proof, we first introduce the following
Lemmas.

\begin{lemma}\label{lem:weight_updates}
    Assume that Assumption \ref{ass:bounded_gradient} holds and that $0 < \beta \gamma <1$. Then for any $t = 1,\dots,T$, the weight updates in Algorithm \ref{alg:moon} satisfy the following inequality
    \begin{equation}
        \| \boldsymbol{z}_{t} - \boldsymbol{z}_{t+1} \|_{\infty} \leq \beta H.
    \end{equation}
\end{lemma}

\begin{proof}
The softmax mapping is $1/2$-Lipschitz with respect to the
$\ell_\infty$ norm. Indeed, for any
$\boldsymbol{x},\boldsymbol{y}\in\mathbb{R}^m$,
\[
\left\|
\operatorname{softmax}(\boldsymbol{x})
-
\operatorname{softmax}(\boldsymbol{y})
\right\|_\infty
\le
\frac{1}{2}
\|\boldsymbol{x}-\boldsymbol{y}\|_\infty.
\]
Consequently,
\begin{align}
\|\boldsymbol{z}_{t+1}-\boldsymbol{z}_t\|_\infty
&\le
\frac{1}{2}
\|\boldsymbol{\xi}_{t+1}-\boldsymbol{\xi}_t\|_\infty
\nonumber\\
&=
\frac{\beta}{2}
\|\boldsymbol{\delta}_t+\gamma\boldsymbol{\xi}_t\|_\infty.
\label{eq:softmax_weight_difference}
\end{align}

We first bound $\boldsymbol{\delta}_t$. By the duality between the
spectral norm and the nuclear norm,
\begin{align}
\|\boldsymbol{\delta}_t\|_\infty
&=
\max_{i\in[m]}
\left|
\left\langle
\boldsymbol{W}_t,
\nabla\ell^i(\boldsymbol{\Theta}_t)
\right\rangle
\right|
\nonumber\\
&\le
\|\boldsymbol{W}_t\|_{S_\infty}
\max_{i\in[m]}
\|\nabla\ell^i(\boldsymbol{\Theta}_t)\|_{S_1}
\le H.
\label{eq:delta_bound}
\end{align}
Here, $\|\boldsymbol{W}_t\|_{S_\infty}\le1$ follows from the
polar-factor construction in Algorithm~\ref{alg:moon}. In particular,
if
\[
\boldsymbol{W}_t=\boldsymbol{U}_t\boldsymbol{V}_t^\top
\]
is obtained from a compact singular value decomposition, then
$\|\boldsymbol{W}_t\|_{S_\infty}\le1$.

Next, the update of $\boldsymbol{\xi}_t$ can be written as
\[
\boldsymbol{\xi}_{t+1}
=
(1-\beta\gamma)\boldsymbol{\xi}_t
-\beta\boldsymbol{\delta}_t.
\]
Since $0<\beta\gamma<1$, the coefficient
$1-\beta\gamma$ is nonnegative. Thus, using
\eqref{eq:delta_bound},
\begin{align}
\|\boldsymbol{\xi}_{t+1}\|_\infty
&\le
(1-\beta\gamma)\|\boldsymbol{\xi}_t\|_\infty
+\beta\|\boldsymbol{\delta}_t\|_\infty
\nonumber\\
&\le
(1-\beta\gamma)\|\boldsymbol{\xi}_t\|_\infty+\beta H.
\label{eq:xi_recursion}
\end{align}
Unrolling \eqref{eq:xi_recursion} and using
$\boldsymbol{\xi}_1=\boldsymbol{0}$ gives
\begin{align}
\|\boldsymbol{\xi}_t\|_\infty
&\le
\beta H
\sum_{s=0}^{t-2}(1-\beta\gamma)^s
\nonumber\\
&=
\frac{H}{\gamma}
\left[1-(1-\beta\gamma)^{t-1}\right]
\le
\frac{H}{\gamma}.
\label{eq:xi_uniform_bound}
\end{align}

Since $z_t=\operatorname{Softmax}(\boldsymbol{\xi}_t)$, the $1/2$-Lipschitz continuity of softmax under the $\ell_\infty$ norm bounds the change in $z_t$ by the corresponding change in the logits. Applying the triangle inequality to the logit update and substituting $\|\boldsymbol{\delta}_t\|_\infty\leq H$ from (22) and $\|\boldsymbol{\xi}_t\|_\infty\leq H/\gamma$ from (24), we obtain
\begin{align*}
\|\boldsymbol{z}_{t+1}-\boldsymbol{z}_t\|_\infty
&\le
\frac{\beta}{2}
\left(
\|\boldsymbol{\delta}_t\|_\infty
+\gamma\|\boldsymbol{\xi}_t\|_\infty
\right)\\
&\le
\frac{\beta}{2}(H+H)
=
\beta H.
\end{align*}
This completes the proof.
\end{proof}

\begin{lemma}\label{lem:momentum_error_deterministic}
    Assume that Assumption \ref{ass:L-smoothness} and \ref{ass:bounded_gradient} hold, and that $0 < \beta \gamma <1$. Then for any $t = 1,\dots,T$, the weight updates in Algorithm \ref{alg:moon} satisfy the following inequality 
    \begin{equation}
        \begin{aligned}
            & \|\boldsymbol{M}_t - \sum_{i=1}^m z_{i,t}\nabla\ell^i(\boldsymbol{\Theta}_t)\|_{S_1} \\
            &\leq (1-\mu)^{t-1}\|\sum_{i=1}^m z_{i,1}\nabla\ell^i(\boldsymbol{\Theta}_{1})\|_{S_1} + \frac{L\alpha}{\mu} + \frac{mH^2}{\mu}\beta. 
        \end{aligned}
    \end{equation}
\end{lemma}

\begin{proof}
Let
\[
\boldsymbol{G}_t:=\sum_{i=1}^m z_{i,t}\nabla \ell_i(\boldsymbol{\Theta}_t)
\quad\text{and}\quad
\boldsymbol{E}_t:=\boldsymbol{M}_t-\boldsymbol{G}_t.
\]
Using the momentum update \(\boldsymbol{M}_t=(1-\mu)\boldsymbol{M}_{t-1}+\mu \boldsymbol{G}_t\), we obtain, for \(t\ge 2\),
\[
\boldsymbol{E}_t=(1-\mu)\bigl(\boldsymbol{E}_{t-1}+\boldsymbol{G}_{t-1}-\boldsymbol{G}_t\bigr).
\]
Iterating this recursion and using \(\boldsymbol{M}_0=0\) gives
\[
\|\boldsymbol{E}_t\|_{S_1}
\le
(1-\mu)^t\|\boldsymbol{G}_1\|_{S_1}
+
\sum_{j=2}^t(1-\mu)^{t-j+1}
\|\boldsymbol{G}_{j-1}-\boldsymbol{G}_j\|_{S_1}.
\]
Moreover,
\begin{align*}
    &\boldsymbol{G}_j-\boldsymbol{G}_{j-1}
    =
    \sum_{i=1}^m z_{i,j}
    \bigl(\nabla\ell_i(\boldsymbol{\Theta}_j)-\nabla\ell_i(\boldsymbol{\Theta}_{j-1})\bigr) \\
    &+
    \sum_{i=1}^m(z_{i,j}-z_{i,j-1})
    \nabla\ell_i(\boldsymbol{\Theta}_{j-1}).
\end{align*}
By Assumptions \ref{ass:L-smoothness}--\ref{ass:bounded_gradient}, Lemma \ref{lem:weight_updates}, and
\(\|\boldsymbol{\Theta}_j-\boldsymbol{\Theta}_{j-1}\|_{S_\infty}\le\alpha\), we have
\[
\|\boldsymbol{G}_j-\boldsymbol{G}_{j-1}\|_{S_1}
\le
L\alpha+mH^2\beta.
\]
Since
\[
\sum_{j=2}^t(1-\mu)^{t-j+1}\le\frac{1}{\mu},
\]
it follows that
\begin{align*}
    &\left\|\boldsymbol{M}_t-\sum_{i=1}^mz_{i,t}\nabla\ell_i(\boldsymbol{\Theta}_t)\right\|_{S_1} \\
    &\le(1-\mu)^{t-1}\left\|\sum_{i=1}^m z_{i,1}\nabla\ell_i(\boldsymbol{\Theta}_1)\right\|_{S_1} +\frac{L\alpha}{\mu}+\frac{mH^2\beta}{\mu}.
\end{align*}
\end{proof}

We are now ready to prove Theorem~\ref{thm:deterministic_convergence_app}. 

\begin{theorem}\label{thm:non_convex_det}
    Suppose Assumption \ref{ass:L-smoothness}--\ref{ass:bounded_function} hold, let $\ell^i:\mathbb{R}^{p \times q}\rightarrow\mathbb{R}$ be non-convex for each $i\in [m]$. Suppose that $0<\beta\gamma<1$, then the iterations of Algorithm \ref{alg:moon} satisfy the following inequality 
    \begin{equation}
    \begin{aligned}
        &\frac{1}{T}\sum_{t=1}^T
        \min_{\boldsymbol{z}_t^\ast \in \Delta_m}
        \|\nabla L(\boldsymbol{\Theta}_t)
        \boldsymbol{z}_t^\ast \|_{S_1} \\
        &\leq \frac{2B}{\alpha T} + \frac{2H}{\mu T} + \frac{2L\alpha}{\mu} + \frac{2mH^2}{\mu}\beta
        + mHB\frac{\beta}{\alpha}
        + \frac{L}{2} \alpha. 
    \end{aligned}
    \end{equation}
\end{theorem}

\begin{proof}
Define
\[
\boldsymbol{G}_t:=\sum_{i=1}^m z_{i,t}\nabla\ell_i(\boldsymbol{\Theta}_t).
\]
By \(L\)-smoothness and \(\boldsymbol{\Theta}_{t+1}=\boldsymbol{\Theta}_t-\alpha \boldsymbol{W}_t\),
\[
\sum_{i=1}^m z_{i,t}\ell_i(\boldsymbol{\Theta}_{t+1})
\le
\sum_{i=1}^m z_{i,t}\ell_i(\boldsymbol{\Theta}_t)
-\alpha\langle \boldsymbol{G}_t,\boldsymbol{W}_t\rangle
+\frac{L\alpha^2}{2}.
\]
Since \(\boldsymbol{W}_t\) is the polar factor of \(\boldsymbol{M}_t\),
\(\langle \boldsymbol{M}_t,\boldsymbol{W}_t\rangle=\|\boldsymbol{M}_t\|_{S_1}\) and
\(\|\boldsymbol{W}_t\|_{S_\infty}\le1\). Therefore,
\[
-\alpha\langle \boldsymbol{G}_t,\boldsymbol{W}_t\rangle
\le
-\alpha\|\boldsymbol{G}_t\|_{S_1}
+2\alpha\|\boldsymbol{G}_t-\boldsymbol{M}_t\|_{S_1}.
\]
After adding and subtracting
\(\sum_i z_{i,t+1}\ell_i(\boldsymbol{\Theta}_{t+1})\) and rearranging, we obtain
\[
\begin{aligned}
\|\boldsymbol{G}_t\|_{S_1}
\le {}&
\frac{1}{\alpha}
\left(
\sum_{i=1}^m z_{i,t}\ell_i(\boldsymbol{\Theta}_t)
-
\sum_{i=1}^m z_{i,t+1}\ell_i(\boldsymbol{\Theta}_{t+1})
\right)\\
&+\frac{1}{\alpha}
\sum_{i=1}^m
(z_{i,t+1}-z_{i,t})\ell_i(\boldsymbol{\Theta}_{t+1})
+\frac{L\alpha}{2} \\
&+2\|\boldsymbol{G}_t-\boldsymbol{M}_t\|_{S_1}.
\end{aligned}
\]
Lemma~\ref{lem:weight_updates} imply
\[
\left|
\sum_{i=1}^m
(z_{i,t+1}-z_{i,t})\ell_i(\boldsymbol{\Theta}_{t+1})
\right|
\le mBH\beta.
\]
Applying Lemma~\ref{lem:momentum_error_deterministic}, summing over \(t=1,\ldots,T\), and using
\(\|\boldsymbol{G}_1\|_{S_1}\le H\), we obtain
\[
\begin{aligned}
&\sum_{t=1}^T\|\boldsymbol{G}_t\|_{S_1}
\le
\frac{2B}{\alpha}
+\frac{2H}{\mu}
+\frac{2TL\alpha}{\mu}
+\frac{2TmH^2\beta}{\mu} \\
&+\frac{TmHB\beta}{\alpha}
+\frac{TL\alpha}{2}.
\end{aligned}
\]
Finally, since $\boldsymbol{z}_t\in\Delta_m$, it is a feasible point of the minimization problem. Therefore,
\[
\min_{z\in\Delta_m}
\left\|
\sum_{i=1}^m z_i\nabla\ell_i(\boldsymbol{\Theta}_t)
\right\|_{S_1}
\le \|\boldsymbol{G}_t\|_{S_1}.
\]
Dividing the preceding inequality by \(T\) yields the claimed result.
\end{proof}

There is an immediate corollary of the above Theorem. 

\begin{corollary}\label{cor:convergence_rate_nonconvex_det}
    Setting $\alpha=\sqrt{\frac{B}{LT}}$, $\beta=\frac{1}{mHT}$ and $0<\gamma<mHT$ in \cref{thm:non_convex_det}, then, for any fixed $\mu \in (0,1]$, the following holds 
    \begin{equation}
        \begin{aligned}
            &\frac{1}{T}\sum_{t=1}^T \min_{\boldsymbol{z}_t^\ast \in \Delta_m}\|\nabla L(\boldsymbol{\Theta_t})\boldsymbol{z}_t^\ast \|_{S_1} \\
            &\leq O(\sqrt{\frac{LB}{T}} + \frac{1}{T}) = O(1/\sqrt{T}). 
        \end{aligned}
    \end{equation}
\end{corollary}

\begin{proof}
    Since $\alpha=\sqrt{\frac{B}{LT}}$, $\beta=\frac{1}{mHT}$ and $0<\gamma<mHT$ in Theorem \ref{thm:non_convex_det}, we have 
    \begin{align*}
        &\frac{1}{T}\sum_{t=1}^T
        \min_{\boldsymbol{z}_t^\ast \in \Delta_m}
        \|\nabla L(\boldsymbol{\Theta_t})
        \boldsymbol{z}_t^\ast \|_{S_1} \\
        &\leq \frac{2B}{\alpha T} + \frac{2H}{\mu T} + \frac{2L\alpha}{\mu} + \frac{2mH^2}{\mu}\beta
        + mHB\frac{\beta}{\alpha}
        + \frac{L}{2} \alpha \\
        &= (\frac{7}{2}+\frac{2}{\mu})\sqrt{\frac{LB}{T}} + \frac{4H}{\mu}\cdot\frac{1}{T}\\
        &= O(\sqrt{\frac{LB}{T}} + \frac{1}{T}) = O(1/\sqrt{T}). 
    \end{align*}
\end{proof}

\subsection{Convergence Analysis for Stochastic Gradient}
\label{section:algo_stochastic_case}

\begin{theorem}[Stochastic Convergence]\label{thm:stochastic_convergence_app}
    Under the same assumption as Theorem~\ref{thm:deterministic_convergence}. We assume that the stochastic gradient estimator is unbiased and has variance bounded by $\sigma$. Choosing $\alpha = \sqrt{\frac{\mu B}{TL}}$, $\beta = \frac{1}{(H+\sqrt{m}\rho\sigma)mT}$, $0<\gamma<(H+\sqrt{m}\rho\sigma)mT$ and $\mu = \min\{ \frac{\sqrt{LB}}{\sigma\rho\sqrt{T}},1\}$, where $\rho := \sqrt{\min \{p,q\}}$, then the iterations of Algorithm \ref{alg:moon} satisfy 
    \begin{equation}
            \frac{1}{T}\sum_{t=1}^T \mathbb{E}\left[\min_{\boldsymbol{z}_t^\ast \in \Delta_m}\|\nabla L(\boldsymbol{\Theta_t})\boldsymbol{z}_t^\ast \|_{S_1}\right] \leq \mathcal{O}(1/T^{1/4}). 
    \end{equation}
\end{theorem}

We now establish the convergence of MOON in the stochastic setting.
Let $\mathcal F_t$ denote the sigma-algebra containing all information
available immediately before sampling $\zeta_t$. In particular,
$\boldsymbol{\Theta}_t$ is $\mathcal F_t$-measurable. For every
$i\in[m]$ and $t\geq1$, assume that, almost surely,
\[
\mathbb E_t
\left[
g^i(\boldsymbol{\Theta}_t;\zeta_t)
\right]
=
\nabla\ell^i(\boldsymbol{\Theta}_t),
\]
and
\[
\mathbb E_t
\left[
\left\|
g^i(\boldsymbol{\Theta}_t;\zeta_t)
-
\nabla\ell^i(\boldsymbol{\Theta}_t)
\right\|_F^2
\right]
\leq \sigma^2,
\]
where
$\mathbb E_t[\cdot]:=\mathbb E[\cdot\mid\mathcal F_t]$
and $\sigma>0$ is independent of $t$ and $T$.

No independence is required among the stochastic-gradient errors
associated with different objectives at the same iteration. Therefore,
the stochastic gradients may be computed using the same sample
$\zeta_t$. Independence across iterations is not required either,
provided that the above conditional assumptions hold.

\begin{lemma}\label{lem:weight_updates_stochastic}
    Assume that Assumption \ref{ass:bounded_gradient} holds, and that $0 < \beta \gamma <1$. Then for any $t = 1,\dots,T$, the weight updates in Algorithm \ref{alg:moon} satisfy the following inequality
    \begin{equation}
        \mathbb{E} \left[ \| \boldsymbol{z}_{t} - \boldsymbol{z}_{t+1} \|_{\infty} \right] \leq \beta (H+\sqrt{m}\rho\sigma).
    \end{equation}
\end{lemma}

\begin{proof}
Let
\[
\boldsymbol{\eta}_{i,t}
:=
g^i(\boldsymbol{\Theta}_t;\zeta_t)
-
\nabla\ell^i(\boldsymbol{\Theta}_t)
\]
denote the stochastic-gradient error.

The softmax mapping is \(1/2\)-Lipschitz with respect to the
\(\ell_\infty\)-norm. Therefore,
\begin{align}
\|\boldsymbol{z}_{t+1}-\boldsymbol{z}_t\|_\infty
&\le
\frac{1}{2}
\|\boldsymbol{\xi}_{t+1}-\boldsymbol{\xi}_t\|_\infty
\nonumber\\
&=
\frac{\beta}{2}
\|
\boldsymbol{\delta}_t+\gamma\boldsymbol{\xi}_t
\|_\infty.
\label{eq:stochastic_softmax_bound}
\end{align}

We first bound \(\mathbb{E}\|\boldsymbol{\delta}_t\|_\infty\).
By the duality between the spectral and nuclear norms and the
almost-sure bound
\(\|\boldsymbol{W}_t\|_{S_\infty}\le1\),
\begin{align}
\|\boldsymbol{\delta}_t\|_\infty
&=
\max_{i\in[m]}
\left|
\left\langle
\boldsymbol{W}_t,
g^i(\boldsymbol{\Theta}_t;\zeta_t)
\right\rangle
\right|
\nonumber\\
&\le
\max_{i\in[m]}
\left\|
g^i(\boldsymbol{\Theta}_t;\zeta_t)
\right\|_{S_1}
\nonumber\\
&\le
\max_{i\in[m]}
\|\nabla\ell^i(\boldsymbol{\Theta}_t)\|_{S_1}
+
\max_{i\in[m]}
\|\boldsymbol{\eta}_{i,t}\|_{S_1}
\nonumber\\
&\le
H+
\max_{i\in[m]}
\|\boldsymbol{\eta}_{i,t}\|_{S_1}.
\label{eq:stochastic_delta_pointwise}
\end{align}

For every matrix
\(\boldsymbol{A}\in\mathbb{R}^{p\times q}\),
\[
\|\boldsymbol{A}\|_{S_1}
\le
\sqrt{\min\{p,q\}}\,
\|\boldsymbol{A}\|_F
=
\rho\|\boldsymbol{A}\|_F.
\]
Consequently,
\begin{align}
\mathbb{E}_t
\left[
\max_{i\in[m]}
\|\boldsymbol{\eta}_{i,t}\|_{S_1}
\right]
&\le
\rho\,
\mathbb{E}_t
\left[
\max_{i\in[m]}
\|\boldsymbol{\eta}_{i,t}\|_F
\right]
\nonumber\\
&\le
\rho\,
\mathbb{E}_t
\left[
\left(
\sum_{i=1}^m
\|\boldsymbol{\eta}_{i,t}\|_F^2
\right)^{1/2}
\right]
\nonumber\\
&\le
\rho
\left(
\sum_{i=1}^m
\mathbb{E}_t
\|\boldsymbol{\eta}_{i,t}\|_F^2
\right)^{1/2}
\nonumber\\
&\le
\rho\sqrt{m}\,\sigma.
\label{eq:max_stochastic_error}
\end{align}

Moreover, by $L$-smoothness and $\Theta_{t+1}=\Theta_t-\alpha W_t$, the loss-difference approximation satisfies
\[
\left|
\frac{\ell^i(\Theta_t)-\ell^i(\Theta_{t+1})}{\alpha}
-\left\langle\nabla\ell^i(\Theta_t),W_t\right\rangle
\right|
\leq
\frac{L\alpha}{2}\|W_t\|_{S_\infty}^2
\leq
\frac{L\alpha}{2},
\]
where the last inequality follows from $\|W_t\|_{S_\infty}\leq 1$.

The penultimate inequality follows from Jensen's inequality.
Combining \eqref{eq:stochastic_delta_pointwise} and
\eqref{eq:max_stochastic_error}, and then taking total expectation,
gives
\begin{equation}
\mathbb{E}
\|\boldsymbol{\delta}_t\|_\infty
\le
H+\rho\sqrt{m}\,\sigma.
\label{eq:expected_stochastic_delta}
\end{equation}

For brevity, define
\[
C_\sigma:=H+\rho\sqrt{m}\,\sigma.
\]
The dual-logit update can be written as
\[
\boldsymbol{\xi}_{t+1}
=
(1-\beta\gamma)\boldsymbol{\xi}_t
-
\beta\boldsymbol{\delta}_t.
\]
Since \(0<\beta\gamma<1\), we have
\begin{align}
\mathbb{E}\|\boldsymbol{\xi}_{t+1}\|_\infty
&\le
(1-\beta\gamma)
\mathbb{E}\|\boldsymbol{\xi}_t\|_\infty
+
\beta
\mathbb{E}\|\boldsymbol{\delta}_t\|_\infty
\nonumber\\
&\le
(1-\beta\gamma)
\mathbb{E}\|\boldsymbol{\xi}_t\|_\infty
+
\beta C_\sigma.
\label{eq:stochastic_xi_recursion}
\end{align}

Unrolling \eqref{eq:stochastic_xi_recursion} and using
\(\boldsymbol{\xi}_1=\boldsymbol{0}\), we obtain
\begin{align}
\mathbb{E}\|\boldsymbol{\xi}_t\|_\infty
&\le
\beta C_\sigma
\sum_{s=0}^{t-2}(1-\beta\gamma)^s
\nonumber\\
&=
\frac{C_\sigma}{\gamma}
\left[
1-(1-\beta\gamma)^{t-1}
\right]
\le
\frac{C_\sigma}{\gamma}.
\label{eq:expected_xi_bound}
\end{align}

Finally, taking expectations in \eqref{eq:stochastic_softmax_bound} and applying the triangle inequality separate the contributions of $\boldsymbol{\delta}_t$ and $\boldsymbol{\xi}_t$. By \eqref{eq:expected_stochastic_delta} and \eqref{eq:expected_xi_bound}, these two contributions are bounded by $C_\sigma$ and $\gamma(C_\sigma/\gamma)=C_\sigma$, respectively. Therefore, we obtain
\begin{align*}
\mathbb{E}
\|\boldsymbol{z}_{t+1}-\boldsymbol{z}_t\|_\infty
&\le
\frac{\beta}{2}
\left(
\mathbb{E}\|\boldsymbol{\delta}_t\|_\infty
+
\gamma\mathbb{E}\|\boldsymbol{\xi}_t\|_\infty
\right)\\
&\le
\frac{\beta}{2}
(C_\sigma+C_\sigma)\\
&=
\beta
\left(
H+\rho\sqrt{m}\,\sigma
\right).
\end{align*}
This completes the proof.
\end{proof}

\begin{lemma}\label{lem:momentum_error_stochastic}
Assume that Assumptions \ref{ass:L-smoothness} and \ref{ass:bounded_gradient} hold. Suppose that $\boldsymbol{M}_0=\boldsymbol{0}$, $0<\mu\leq 1$, and $0<\beta\gamma<1$. Then, for any $t=1,\ldots,T$, the iterates generated by Algorithm \ref{alg:moon} satisfy
\begin{equation}
\begin{aligned}
&\mathbb{E}\left[\left\|\boldsymbol{M}_t-\sum_{i=1}^m z_{i,t}\nabla\ell^i(\boldsymbol{\Theta}_t)\right\|_{S_1}\right] \\
&\leq (1-\mu)^{t-1}\left\|\sum_{i=1}^m z_{i,1}\nabla\ell^i(\boldsymbol{\Theta}_1)\right\|_{S_1}+\frac{L\alpha}{\mu} \\
&+\sqrt{\mu}\rho\sigma+\frac{(H+\sqrt{m}\rho\sigma)mH}{\mu}\beta.
\end{aligned}
\end{equation}
\end{lemma}

\begin{proof}
Define the weighted true gradient and the aggregated stochastic-gradient noise as
\begin{align*}
\boldsymbol{G}_t&:=\sum_{i=1}^m z_{i,t}\nabla\ell^i(\boldsymbol{\Theta}_t),\\
\boldsymbol{\epsilon}_t&:=\sum_{i=1}^m z_{i,t}\left(\boldsymbol{g}^i(\boldsymbol{\Theta}_t;\zeta_t)-\nabla\ell^i(\boldsymbol{\Theta}_t)\right).
\end{align*}
The momentum update can then be written as
\[
\boldsymbol{M}_t=(1-\mu)\boldsymbol{M}_{t-1}+\mu(\boldsymbol{G}_t+\boldsymbol{\epsilon}_t).
\]
Unrolling this recursion gives
\[
\boldsymbol{M}_t=(1-\mu)^t\boldsymbol{M}_0+\mu\sum_{j=1}^t(1-\mu)^{t-j}\left(\boldsymbol{G}_j+\boldsymbol{\epsilon}_j\right).
\]
Subtracting $\boldsymbol{G}_t$ and rearranging the deterministic terms yields
\begin{align*}
\boldsymbol{M}_t-\boldsymbol{G}_t
={}&(1-\mu)^t(\boldsymbol{M}_0-\boldsymbol{G}_1)\nonumber\\
&+\sum_{j=2}^t(1-\mu)^{t-j+1}(\boldsymbol{G}_{j-1}-\boldsymbol{G}_j)\nonumber\\
&+\mu\sum_{j=1}^t(1-\mu)^{t-j}\boldsymbol{\epsilon}_j.
\label{eq:stochastic_momentum_decomposition}
\end{align*}

We first bound the variation of the weighted true gradient. For every $j\geq 2$,
\begin{align*}
\boldsymbol{G}_j-\boldsymbol{G}_{j-1}
={}&\sum_{i=1}^m z_{i,j}\left(\nabla\ell^i(\boldsymbol{\Theta}_j)-\nabla\ell^i(\boldsymbol{\Theta}_{j-1})\right)\\
&+\sum_{i=1}^m(z_{i,j}-z_{i,j-1})\nabla\ell^i(\boldsymbol{\Theta}_{j-1}).
\end{align*}
By Assumptions \ref{ass:L-smoothness} and \ref{ass:bounded_gradient}, $\|\boldsymbol{W}_{j-1}\|_{S_\infty}\leq 1$, and the stochastic weight-change bound
\[
\mathbb{E}\left[\|\boldsymbol{z}_j-\boldsymbol{z}_{j-1}\|_\infty\right]\leq\beta(H+\sqrt{m}\rho\sigma),
\]
we obtain
\begin{align*}
\mathbb{E}\left[\|\boldsymbol{G}_j-\boldsymbol{G}_{j-1}\|_{S_1}\right]
&\leq L\mathbb{E}\left[\|\boldsymbol{\Theta}_j-\boldsymbol{\Theta}_{j-1}\|_{S_\infty}\right] \\
&+H\sum_{i=1}^m\mathbb{E}\left[|z_{i,j}-z_{i,j-1}|\right]\nonumber\\
&\leq L\alpha+(H+\sqrt{m}\rho\sigma)mH\beta.
\label{eq:weighted_gradient_variation}
\end{align*}

We next bound the stochastic-noise term. Let $\mathcal{F}_j$ denote the history before sampling $\zeta_j$. Since $\boldsymbol{z}_j$ and $\boldsymbol{\Theta}_j$ are $\mathcal{F}_j$-measurable, the conditional unbiasedness assumption gives
\[
\mathbb{E}[\boldsymbol{\epsilon}_j\mid\mathcal{F}_j]=\boldsymbol{0}.
\]
Moreover, by the convexity of the squared Frobenius norm,
\begin{align*}
&\mathbb{E}\left[\|\boldsymbol{\epsilon}_j\|_F^2\mid\mathcal{F}_j\right] \\
&=\mathbb{E}\left[\left\|\sum_{i=1}^m z_{i,j}\left(\boldsymbol{g}^i(\boldsymbol{\Theta}_j;\zeta_j)-\nabla\ell^i(\boldsymbol{\Theta}_j)\right)\right\|_F^2\middle|\mathcal{F}_j\right]\\
&\leq\sum_{i=1}^m z_{i,j}\mathbb{E}\left[\left\|\boldsymbol{g}^i(\boldsymbol{\Theta}_j;\zeta_j)-\nabla\ell^i(\boldsymbol{\Theta}_j)\right\|_F^2\middle|\mathcal{F}_j\right] \leq\sigma^2.
\end{align*}
Because $\{\boldsymbol{\epsilon}_j\}_{j\geq 1}$ is a martingale-difference sequence, the cross terms vanish. Using $\|\boldsymbol{A}\|_{S_1}\leq\rho\|\boldsymbol{A}\|_F$ and Jensen's inequality, we have
\begin{align*}
&\mathbb{E}\left[\left\|\mu\sum_{j=1}^t(1-\mu)^{t-j}\boldsymbol{\epsilon}_j\right\|_{S_1}\right]\nonumber\\
&\leq\rho\sqrt{\mathbb{E}\left[\left\|\mu\sum_{j=1}^t(1-\mu)^{t-j}\boldsymbol{\epsilon}_j\right\|_F^2\right]}\nonumber\\
&\leq\rho\mu\sigma\sqrt{\sum_{j=1}^t(1-\mu)^{2(t-j)}}\nonumber \leq\sqrt{\mu}\rho\sigma.
\label{eq:stochastic_noise_bound}
\end{align*}

Applying the triangle inequality to the preceding decomposition, using $\boldsymbol{M}_0=\boldsymbol{0}$, and taking total expectation separate the momentum error into three parts: the initial error, the accumulated variation of the weighted true gradients, and the accumulated stochastic-gradient noise. The weighted-gradient variation at each step is bounded by $L\alpha+(H+\sqrt{m}\rho\sigma)mH\beta$, while the expected nuclear norm of the accumulated noise is bounded by $\sqrt{\mu}\rho\sigma$. Therefore, we obtain
\begin{align*}
&\mathbb{E}\left[\|\boldsymbol{M}_t-\boldsymbol{G}_t\|_{S_1}\right]
\leq{}(1-\mu)^t\|\boldsymbol{G}_1\|_{S_1}+\sqrt{\mu}\rho\sigma\\
&+\sum_{j=2}^t(1-\mu)^{t-j+1}\left(L\alpha+(H+\sqrt{m}\rho\sigma)mH\beta\right). 
\end{align*}
Finally, since
\[
(1-\mu)^t\leq(1-\mu)^{t-1}
\]
and
\[
\sum_{j=2}^t(1-\mu)^{t-j+1}\leq\frac{1}{\mu},
\]
we conclude that
\[
\begin{aligned}
&\mathbb{E}\left[\|\boldsymbol{M}_t-\boldsymbol{G}_t\|_{S_1}\right]
\leq(1-\mu)^{t-1}\|\boldsymbol{G}_1\|_{S_1}+\frac{L\alpha}{\mu}\\
&+\sqrt{\mu}\rho\sigma+\frac{(H+\sqrt{m}\rho\sigma)mH}{\mu}\beta.
\end{aligned}
\]
Substituting the definition of $\boldsymbol{G}_t$ completes the proof.
\end{proof}

Now we are ready to prove Theorem~\ref{thm:stochastic_convergence_app} 

\begin{theorem}\label{thm:non_convex_stochastic}
    Let Assumption \ref{ass:L-smoothness} to \ref{ass:bounded_function} hold, let $\ell^i:\mathbb{R}^{p\times q}\rightarrow\mathbb{R}$ be non-convex for each $i\in [m]$. Then the iterations of Algorithm \ref{alg:moon} satisfy the following inequality 
    \begin{equation}
    \begin{aligned}
        &\frac{1}{T}\sum_{t=1}^T \mathbb{E}\left[\min_{\boldsymbol{z}_t^\ast \in \Delta_m}
        \|\nabla L(\boldsymbol{\Theta_t})
        \boldsymbol{z}_t^\ast \|_{S_1}\right] \\ 
        &\leq \frac{2B}{\alpha T} + \frac{2H}{\mu T} + (H+\sqrt{m}\rho\sigma)mB\frac{\beta}{\alpha} + \frac{L}{2}\alpha \\
        &+ \frac{2L\alpha}{\mu} + 2\sqrt{\mu}\rho\sigma + \frac{2mH(H+\sqrt{m}\rho\sigma)}{\mu}\beta.
    \end{aligned}
    \end{equation}
\end{theorem}

\begin{proof}
    For brevity, define
        \[
        \boldsymbol{G}_t
        :=
        \sum_{i=1}^m
        z_{i,t}\nabla\ell^i(\boldsymbol{\Theta}_t).
        \]
        
    From the L-smoothness of each objective function, we have
    \begin{align*}
        &\sum_{i=1}^m z_{i,t} \ell^i(\boldsymbol{\Theta}_{t+1})
        \leq \sum_{i=1}^m z_{i,t} \ell^i(\boldsymbol{\Theta}_t) + \frac{L}{2}
        \| \boldsymbol{\Theta}_{t+1}
        - \boldsymbol{\Theta}_t \|_{S_{\infty}}^2\\
        &\quad + \left\langle
        \sum_{i=1}^m z_{i,t} \nabla\ell^i(\boldsymbol{\Theta}_t),
        \boldsymbol{\Theta}_{t+1} - \boldsymbol{\Theta}_t
        \right\rangle \\
        &\leq \sum_{i=1}^m z_{i,t} \ell^i(\boldsymbol{\Theta}_t) + \frac{L}{2}
        \| \boldsymbol{\Theta}_{t+1}
        - \boldsymbol{\Theta}_t \|_{S_{\infty}}^2 + \left\langle
        \boldsymbol{M}_t,
        \boldsymbol{\Theta}_{t+1} - \boldsymbol{\Theta}_t
        \right\rangle \\
        &\quad + \left\langle
        \sum_{i=1}^m z_{i,t}\nabla\ell^i(\boldsymbol{\Theta}_t) - \boldsymbol{M}_t,
        \boldsymbol{\Theta}_{t+1} - \boldsymbol{\Theta}_t
        \right\rangle.  
    \end{align*}
    
    Since $\boldsymbol{\Theta}_{t+1}=\boldsymbol{\Theta}_t - \alpha\boldsymbol{W}_t$, we have 
    \begin{align*}
        &\left\langle
        \boldsymbol{M}_t,
        \boldsymbol{\Theta}_{t+1} - \boldsymbol{\Theta}_t
        \right\rangle \\
        &= \langle \boldsymbol{M}_t,
        -\alpha \boldsymbol{W}_t \rangle = -\alpha \| \boldsymbol{M}_t \|_{S_1}. 
    \end{align*}

    The spectral-nuclear norm duality bounds the last inner product by
    \[
    \left\|
    \boldsymbol{G}_t-\boldsymbol{M}_t
    \right\|_{S_1}
    \left\|
    \boldsymbol{\Theta}_{t+1}-\boldsymbol{\Theta}_t
    \right\|_{S_\infty}.
    \]
    Moreover,
    $\|\boldsymbol{\Theta}_{t+1}-\boldsymbol{\Theta}_t\|_{S_\infty}
    =\alpha\|\boldsymbol{W}_t\|_{S_\infty}\leq\alpha$, and the reverse triangle inequality gives
    \[
    \|\boldsymbol{M}_t\|_{S_1}
    \geq
    \|\boldsymbol{G}_t\|_{S_1}
    -
    \|\boldsymbol{G}_t-\boldsymbol{M}_t\|_{S_1}.
    \]
    Substituting these two bounds into the smoothness inequality yields
    \begin{align*}
        &\sum_{i=1}^m z_{i,t} \ell^i(\boldsymbol{\Theta}_{t+1})
        \leq \sum_{i=1}^m z_{i,t} \ell^i(\boldsymbol{\Theta}_t) + \frac{L}{2}\alpha^2 -\alpha \left\| \boldsymbol{M}_t \right\|_{S_1} \\
        &+ \|
        \sum_{i=1}^m z_{i,t}\nabla\ell^i(\boldsymbol{\Theta}_t) - \boldsymbol{M}_t\|_{S_1} \| \boldsymbol{\Theta}_{t+1} - \boldsymbol{\Theta}_t
        \|_{S_\infty} \\
        &\leq \sum_{i=1}^m z_{i,t} \ell^i(\boldsymbol{\Theta}_t) + \frac{L}{2}\alpha^2 -\alpha \left\| \sum_{i=1}^m z_{i,t} \nabla\ell^i(\boldsymbol{\Theta}_t) \right\|_{S_1} \\
        &+ 2\alpha \left\| \sum_{i=1}^m z_{i,t} \nabla\ell^i(\boldsymbol{\Theta}_t) - \boldsymbol{M}_t \right\|_{S_1}. 
    \end{align*}

    Rearranging the preceding inequality, adding and subtracting
    $\sum_{i=1}^m z_{i,t+1}\ell^i(\boldsymbol{\Theta}_{t+1})$ to create a telescoping term, and taking total expectation gives the first inequality below. For the second inequality, we use
    \[
    \begin{aligned}
    &\mathbb{E}\left[
    \left|
    \sum_{i=1}^m
    (z_{i,t+1}-z_{i,t})
    \ell^i(\boldsymbol{\Theta}_{t+1})
    \right|
    \right]\\
    &\leq
    mB\mathbb{E}
    \left[
    \|\boldsymbol{z}_{t+1}-\boldsymbol{z}_t\|_\infty
    \right]
    \leq
    mB\beta(H+\sqrt{m}\rho\sigma),
    \end{aligned}
    \]
    together with Lemma~\ref{lem:momentum_error_stochastic}. Therefore,
    \begin{align*}
        &\mathbb{E}\left[\left\|
        \sum_{i=1}^m z_{i,t}
        \nabla\ell^i(\boldsymbol{\Theta}_{t})
        \right\|_{S_1}\right] \\
        &\leq \frac{1}{\alpha}\mathbb{E}\left[
        \sum_{i=1}^m z_{i,t} \ell^i(\boldsymbol{\Theta}_t)
        - \sum_{i=1}^m z_{i,t+1}
        \ell^i(\boldsymbol{\Theta}_{t+1}) \right] \\
        &\quad + \frac{1}{\alpha}\mathbb{E}\left[
        \sum_{i=1}^m (z_{i,t+1} -z_{i,t})
        \ell^i(\boldsymbol{\Theta}_{t+1}) \right]
        + \frac{L}{2} \alpha \\
        &\quad + 2\mathbb{E}\left[\left\| \sum_{i=1}^m z_{i,t} \nabla\ell^i(\boldsymbol{\Theta}_t) - \boldsymbol{M}_t \right\|_{S_1}\right] \\ 
        &\leq \frac{1}{\alpha}\mathbb{E}\left[
        \sum_{i=1}^m z_{i,t} \ell^i(\boldsymbol{\Theta}_t)
        - \sum_{i=1}^m z_{i,t+1}
        \ell^i(\boldsymbol{\Theta}_{t+1}) \right] \\
        &+ 2(1-\mu)^{t-1}\|\sum_{i=1}^m z_{i,1}\nabla\ell^i(\boldsymbol{\Theta}_{1})\|_{S_1} \\
        &+ (H+\sqrt{m}\rho\sigma)mB\frac{\beta}{\alpha} + \frac{L}{2}\alpha + \frac{2L\alpha}{\mu} \\
        &+ 2\sqrt{\mu}\rho\sigma + \frac{2mH(H+\sqrt{m}\rho\sigma)}{\mu}\beta. 
    \end{align*}

    Summing the preceding inequality from $t=1$ to $T$, the weighted-loss terms telescope. Since $\boldsymbol{z}_t\in\Delta_m$ and $|\ell^i(\boldsymbol{\Theta}_t)|\leq B$, the two endpoint terms are bounded by $2B$. Moreover,
    \[
    \sum_{t=1}^T(1-\mu)^{t-1}\leq\frac{1}{\mu},
    \qquad
    \left\|
    \sum_{i=1}^m
    z_{i,1}\nabla\ell^i(\boldsymbol{\Theta}_1)
    \right\|_{S_1}
    \leq H.
    \]
    Consequently,
    \begin{align*}
        &\sum_{t=1}^T \mathbb{E}\left[\left\|
        \sum_{i=1}^m z_{i,t}
        \nabla\ell^i(\boldsymbol{\Theta}_{t})
        \right\|_{S_1}\right] \\
        &\leq \frac{1}{\alpha}\mathbb{E}\left[
        \sum_{i=1}^m z_{i,1} \ell^i(\boldsymbol{\Theta}_1)
        - \sum_{i=1}^m z_{i,T+1}
        \ell^i(\boldsymbol{\Theta}_{T+1}) \right] \\
        &+ \frac{2H}{\mu} + (H+\sqrt{m}\rho\sigma)TmB\frac{\beta}{\alpha} + \frac{TL}{2}\alpha + \frac{2TL\alpha}{\mu} \\
        &+ 2T\sqrt{\mu}\rho\sigma + \frac{2TmH(H+\sqrt{m}\rho\sigma)}{\mu}\beta \\
        &\leq \frac{2B}{\alpha} + \frac{2H}{\mu} + (H+\sqrt{m}\rho\sigma)TmB\frac{\beta}{\alpha} + \frac{TL}{2}\alpha \\
        &+ \frac{2TL\alpha}{\mu} + 2T\sqrt{\mu}\rho\sigma + \frac{2TmH(H+\sqrt{m}\rho\sigma)}{\mu}\beta. 
    \end{align*}

    Finally, since $\boldsymbol{z}_t\in\Delta_m$ is feasible,
    \[
    \min_{\boldsymbol{z}_t^\ast\in\Delta_m}
    \left\|
    \nabla L(\boldsymbol{\Theta}_t)\boldsymbol{z}_t^\ast
    \right\|_{S_1}
    \leq
    \left\|
    \sum_{i=1}^m
    z_{i,t}\nabla\ell^i(\boldsymbol{\Theta}_t)
    \right\|_{S_1}.
    \]
    Dividing the preceding inequality by $T$ and applying this pointwise bound proves the theorem.
\end{proof}

There is an immediate corollary of the above Theorem \ref{thm:non_convex_stochastic}. 

\begin{corollary}\label{cor:convergence_rate_nonconvex_stochastic}
    Setting $\alpha = \sqrt{\frac{\mu B}{TL}}$, $\beta = \frac{1}{(H+\sqrt{m}\rho\sigma)mT}$, $0<\gamma<(H+\sqrt{m}\rho\sigma)mT$ and $\mu = \min\{ \frac{\sqrt{LB}}{\sigma\rho\sqrt{T}},1\}$ in \cref{thm:non_convex_stochastic},  the following holds 
    \begin{equation}
        \begin{aligned}
            &\frac{1}{T}\sum_{t=1}^T \mathbb{E}\left[\min_{\boldsymbol{z}_t^\ast \in \Delta_m}\|\nabla L(\boldsymbol{\Theta_t})\boldsymbol{z}_t^\ast \|_{S_1}\right] \\ 
            &\leq \mathcal{O}\left( \sqrt{\frac{LB}{T}} + \left( \frac{\rho^2\sigma^2 LB}{T} \right)^{1/4} \right) = \mathcal{O}(1/T^{1/4}). 
        \end{aligned}
    \end{equation}
\end{corollary}

\begin{proof}
    Since $\alpha = \sqrt{\frac{\mu B}{TL}}$, $\beta = \frac{1}{(H+\sqrt{m}\rho\sigma)mT}$, $0<\gamma<(H+\sqrt{m}\rho\sigma)mT$ and $\mu = \min\{ \frac{\sqrt{LB}}{\sigma\rho\sqrt{T}},1\}$ in Theorem \ref{thm:non_convex_stochastic}, we have 
    \begin{align*}
        &\frac{1}{T}\sum_{t=1}^T \mathbb{E}\left[\min_{\boldsymbol{z}_t^\ast \in \Delta_m}
        \|\nabla L(\boldsymbol{\Theta_t})
        \boldsymbol{z}_t^\ast \|_{S_1}\right] \\ 
        &\leq \frac{2B}{\alpha T} + \frac{2H}{\mu T} + (H+\sqrt{m}\rho\sigma)mB\frac{\beta}{\alpha} + \frac{L}{2}\alpha \\
        &+ \frac{2L\alpha}{\mu} + 2\sqrt{\mu}\rho\sigma + \frac{2mH(H+\sqrt{m}\rho\sigma)}{\mu}\beta \\
        &\leq (\frac{4\sigma\rho H}{LB} + \frac{1}{2})\sqrt{\frac{LB}{T}} + 7\left( \frac{\rho^2\sigma^2 LB}{T} \right)^{1/4} \\ 
        &=\mathcal{O}\left( \sqrt{\frac{LB}{T}} + \left( \frac{\rho^2\sigma^2 LB}{T} \right)^{1/4} \right) = \mathcal{O}(1/T^{1/4}). 
    \end{align*}
\end{proof}

\section{Convergence of Convex Case}

\begin{theorem}\label{thm:convex_deterministic}
Let $\Omega\subseteq\mathbb{R}^{p\times q}$ be a nonempty compact
convex set satisfying
\[
    \operatorname{diam}_{S_\infty}(\Omega)
    :=
    \sup_{\boldsymbol{\Theta},\boldsymbol{\Theta}'\in\Omega}
    \|\boldsymbol{\Theta}-\boldsymbol{\Theta}'\|_{S_\infty}
    \leq D.
\]
Suppose that each objective
$\ell^i:\Omega\rightarrow\mathbb{R}$ is convex and differentiable.
Assume that the iterates generated by Algorithm~\ref{alg:moon}
satisfy
\[
    \boldsymbol{\Theta}_t\in\Omega,
    \qquad
    \boldsymbol z_t\in\Delta_m,
    \qquad t=1,\ldots,T.
\]
Then
\begin{equation}
\begin{aligned}
&\frac{1}{T}\sum_{t=1}^T
\max_{\boldsymbol{\Theta}^\ast\in\Omega}
\min_{\boldsymbol z^\ast\in\Delta_m}
{\boldsymbol z^\ast}^{\top}
\left(
\boldsymbol L(\boldsymbol{\Theta}_t)
-
\boldsymbol L(\boldsymbol{\Theta}^\ast)
\right)
\\
&\qquad\leq
\frac{D}{T}\sum_{t=1}^T
\left\|
\sum_{i=1}^m z_{i,t}
\nabla\ell^i(\boldsymbol{\Theta}_t)
\right\|_{S_1}.
\end{aligned}
\end{equation}
\end{theorem}

\begin{proof}
Fix any $t\in\{1,\ldots,T\}$. Since
$\boldsymbol z_t\in\Delta_m$, for every
$\boldsymbol{\Theta}^\ast\in\Omega$, we have
\begin{align}
&\min_{\boldsymbol z^\ast\in\Delta_m}
{\boldsymbol z^\ast}^{\top}
\left(
\boldsymbol L(\boldsymbol{\Theta}_t)
-
\boldsymbol L(\boldsymbol{\Theta}^\ast)
\right)
\nonumber\\
&\qquad\leq
\boldsymbol z_t^{\top}
\left(
\boldsymbol L(\boldsymbol{\Theta}_t)
-
\boldsymbol L(\boldsymbol{\Theta}^\ast)
\right).
\label{eq:choose_zt}
\end{align}
Taking the maximum over
$\boldsymbol{\Theta}^\ast\in\Omega$ on both sides gives
\begin{align}
&\max_{\boldsymbol{\Theta}^\ast\in\Omega}
\min_{\boldsymbol z^\ast\in\Delta_m}
{\boldsymbol z^\ast}^{\top}
\left(
\boldsymbol L(\boldsymbol{\Theta}_t)
-
\boldsymbol L(\boldsymbol{\Theta}^\ast)
\right)
\nonumber\\
&\qquad\leq
\max_{\boldsymbol{\Theta}^\ast\in\Omega}
\boldsymbol z_t^{\top}
\left(
\boldsymbol L(\boldsymbol{\Theta}_t)
-
\boldsymbol L(\boldsymbol{\Theta}^\ast)
\right)
\nonumber\\
&\qquad=
\boldsymbol z_t^{\top}\boldsymbol L(\boldsymbol{\Theta}_t)
-
\min_{\boldsymbol{\Theta}^\ast\in\Omega}
\boldsymbol z_t^{\top}\boldsymbol L(\boldsymbol{\Theta}^\ast).
\label{eq:scalarized_gap}
\end{align}

Because $\Omega$ is compact and each $\ell^i$ is continuous, there
exists
\[
\boldsymbol{\Theta}_t^\ast
\in
\argmin_{\boldsymbol{\Theta}\in\Omega}
\boldsymbol z_t^\top\boldsymbol L(\boldsymbol{\Theta}).
\]
Consequently, the right-hand side of
\eqref{eq:scalarized_gap} is
\[
\sum_{i=1}^m z_{i,t}
\left(
\ell^i(\boldsymbol{\Theta}_t)
-
\ell^i(\boldsymbol{\Theta}_t^\ast)
\right).
\]

By the convexity of each $\ell^i$, we have
\[
\ell^i(\boldsymbol{\Theta}_t)
-
\ell^i(\boldsymbol{\Theta}_t^\ast)
\leq
\left\langle
\nabla\ell^i(\boldsymbol{\Theta}_t),
\boldsymbol{\Theta}_t-\boldsymbol{\Theta}_t^\ast
\right\rangle.
\]
Since $z_{i,t}\geq0$, multiplying by $z_{i,t}$ and summing over
$i$ yields
\begin{align}
&\boldsymbol z_t^\top\boldsymbol L(\boldsymbol{\Theta}_t)
-
\boldsymbol z_t^\top\boldsymbol L(\boldsymbol{\Theta}_t^\ast)
\nonumber\\
&\qquad\leq
\left\langle
\sum_{i=1}^m z_{i,t}
\nabla\ell^i(\boldsymbol{\Theta}_t),
\boldsymbol{\Theta}_t-\boldsymbol{\Theta}_t^\ast
\right\rangle.
\label{eq:convexity_bound}
\end{align}

Using the duality between the nuclear norm and the spectral norm,
\[
|\langle \boldsymbol A,\boldsymbol B\rangle|
\leq
\|\boldsymbol A\|_{S_1}
\|\boldsymbol B\|_{S_\infty},
\]
we obtain
\begin{align}
&\left\langle
\sum_{i=1}^m z_{i,t}
\nabla\ell^i(\boldsymbol{\Theta}_t),
\boldsymbol{\Theta}_t-\boldsymbol{\Theta}_t^\ast
\right\rangle
\nonumber\\
&\qquad\leq
\left\|
\sum_{i=1}^m z_{i,t}
\nabla\ell^i(\boldsymbol{\Theta}_t)
\right\|_{S_1}
\|\boldsymbol{\Theta}_t-\boldsymbol{\Theta}_t^\ast\|_{S_\infty}
\nonumber\\
&\qquad\leq
D
\left\|
\sum_{i=1}^m z_{i,t}
\nabla\ell^i(\boldsymbol{\Theta}_t)
\right\|_{S_1},
\label{eq:diameter_bound}
\end{align}
where the last inequality follows from
$\boldsymbol{\Theta}_t,\boldsymbol{\Theta}_t^\ast\in\Omega$ and
$\operatorname{diam}_{S_\infty}(\Omega)\leq D$.

Combining \eqref{eq:scalarized_gap}--\eqref{eq:diameter_bound}, we
obtain, for every $t$,
\begin{align}
&\max_{\boldsymbol{\Theta}^\ast\in\Omega}
\min_{\boldsymbol z^\ast\in\Delta_m}
{\boldsymbol z^\ast}^{\top}
\left(
\boldsymbol L(\boldsymbol{\Theta}_t)
-
\boldsymbol L(\boldsymbol{\Theta}^\ast)
\right)
\nonumber\\
&\qquad\leq
D
\left\|
\sum_{i=1}^m z_{i,t}
\nabla\ell^i(\boldsymbol{\Theta}_t)
\right\|_{S_1}.
\end{align}
Summing this inequality over $t=1,\ldots,T$ and dividing by $T$
completes the proof.
\end{proof}

\section{Approximation to Avoid Calculating Multiple Gradients}

In this section, we present an approximation of the weight update rule to avoid calculating multiple gradients in \cref{alg:moon}. 

In \cref{alg:moon}, we update $\boldsymbol{z}$ using a single gradient-descent step as follows
\begin{align*}
    \boldsymbol{z}_{t+1}
    &= \boldsymbol{z}_{t}
    -\beta \nabla_{\boldsymbol{z}}
    \left(\frac{1}{2}
    \| \boldsymbol{z}_t^\top
    \nabla \mathcal{L}(\boldsymbol{\Theta}_{t})
    \|_{S_{1}}^{2}\right),
\end{align*}
where $\mathcal{L}(\boldsymbol{\Theta}) = (\ell^1(\Theta), \dots, \ell^m(\Theta))^\top$. Let $[\boldsymbol{a}]_i$ denote the $i$th component of a vector. However, the update rule still necessitates information about the gradients of all the objectives. Therefore, we consider approximating the gradient $\nabla_{\boldsymbol{z}}(\frac{1}{2}\| \boldsymbol{z}_t^\top \nabla \mathcal{L}(\boldsymbol{\Theta}_{t}) \|_{S_{1}}^{2})$ using first-order Taylor’s approximation:
\begin{align*}
    &\left[\nabla_{\boldsymbol{z}}
    \frac{1}{2}\|\boldsymbol{z}_t^{\top}
      \nabla\mathcal{L}(\boldsymbol{\Theta}_t)\|_{S_1}^2
    \right]_i\\
    &= \|\boldsymbol{G}_t\|_{S_1} \left\langle
        \hat{\nabla}\|\boldsymbol{G}_t\|_{S_1},
        \nabla\ell^i(\Theta_t)
      \right\rangle \\
    &= \|\boldsymbol{G}_t\|_{S_1} \left\langle
        \boldsymbol{U}_t\boldsymbol{V}_t^{\top},
        \nabla\ell^i(\Theta_t)
      \right\rangle \\
    &= \|\boldsymbol{G}_t\|_{S_1} \left\langle
        \boldsymbol{W}_t,\nabla\ell^i(\Theta_t)
      \right\rangle \\
    &\approx \frac{\|\boldsymbol{G}_t\|_{S_1}}{\alpha}
      \left[\ell^i(\Theta_t)-\ell^i(\Theta_{t+1})\right], \\
    &\quad i\in[m].
\end{align*}
where $\hat{\nabla}\| \boldsymbol{G}_t \|_{S_{1}}$ represents an element of the subdifferential of the nuclear norm at $\boldsymbol{G}_t$. This identity is exact in the no-momentum case $\mu=1$, where $\boldsymbol{M}_t=\boldsymbol{G}_t$ and $\boldsymbol{W}_t=\operatorname{Polar}(\boldsymbol{G}_t)$.

The factor $\| \boldsymbol{G}_t \|_{S_{1}}$ is a nonnegative scalar that does not change the update direction. Consistent with the practical algorithm, we retain the scale-normalized direction and control the update magnitude through the step size. Then consequently, we have an approximate method for solving $\boldsymbol{z}_{t}$ without backpropagation:
\begin{align}
    \boldsymbol{z}_{t+1}
    &= \boldsymbol{z}_t-\beta\tilde{\boldsymbol{\delta}}_t, \\
    [\tilde{\boldsymbol{\delta}}_t]_i
    &= \frac{1}{\alpha}
      \left[\ell^i(\Theta_t)-\ell^i(\Theta_{t+1})\right],
      \quad i\in[m].
\end{align}

In practical MOON, we use the resulting scale-normalized dual direction in the logit-space update of Algorithm~\ref{alg:moon}, which corresponds to an exponentiated-gradient-type online update on the simplex.

\section{Effect of Finite-Step Newton--Schulz Approximation Error}\label{app:NS_error}

The convergence analysis before is stated in terms of the
exact polar factor of the momentum matrix. In practice, however, Algorithm~\ref{alg:moon} computes the update direction using only a finite number of Newton--Schulz iterations. Consequently, the resulting update direction may differ from an exact polar factor.

In this section, we quantify the effect of this approximation error on the deterministic convergence guarantee. We focus on the deterministic setting because it isolates the
effect of the finite-step Newton--Schulz approximation from
the error caused by stochastic-gradient noise. The stochastic
case follows by the same argument after conditioning on the
history up to iteration $t$ and taking total expectation. We therefore present only the
deterministic proof to avoid repeating the same argument. We show that, under a uniform relative approximation condition, the finite-step Newton--Schulz implementation preserves the $\mathcal{O}(T^{-1/2})$ Pareto-stationarity rate, up to a multiplicative factor depending on the approximation accuracy.

For each iteration $t$, define the weighted gradient matrix
\begin{equation*}
    \boldsymbol G_t
    :=
    \sum_{i=1}^{m}
    z_{i,t}\nabla \ell^i(\boldsymbol\Theta_t),
    \qquad
    \boldsymbol z_t\in\Delta_m,
\end{equation*}
where
\[
    \Delta_m
    :=
    \left\{
        \boldsymbol z\in\mathbb R_+^m:
        \sum_{i=1}^{m}z_i=1
    \right\}.
\]
The momentum matrix is updated according to
\begin{equation*}
    \boldsymbol M_t
    =
    (1-\mu)\boldsymbol M_{t-1}
    +
    \mu\boldsymbol G_t,
    \qquad
    0<\mu\leq1.
\end{equation*}

Let
\[
    \boldsymbol P_t
    \in
    \operatorname{Polar}(\boldsymbol M_t)
\]
denote a polar factor of $\boldsymbol M_t$. In particular,
$\boldsymbol P_t$ satisfies
\begin{equation}
    \|\boldsymbol P_t\|_{S_\infty}\leq1,
    \qquad
    \langle\boldsymbol M_t,\boldsymbol P_t\rangle
    =
    \|\boldsymbol M_t\|_{S_1}.
    \label{eq:NS_exact_polar_properties}
\end{equation}
These properties remain valid when $\boldsymbol M_t$ is rank
deficient; in this case, $\boldsymbol P_t$ should be understood as
any polar factor satisfying
\eqref{eq:NS_exact_polar_properties}.

Let $\widetilde{\boldsymbol W}_t$ denote the direction returned by
the finite-step Newton--Schulz iteration. The actual parameter update
is
\begin{equation}
    \boldsymbol\Theta_{t+1}
    =
    \boldsymbol\Theta_t
    -
    \alpha\widetilde{\boldsymbol W}_t.
    \label{eq:NS_parameter_update}
\end{equation}

We impose the following approximation condition.

\begin{assumption}\label{ass:finite_NS_error}
For every $t\geq1$, there exists a polar factor
$\boldsymbol P_t\in\operatorname{Polar}(\boldsymbol M_t)$ such that
\begin{equation}
    \varepsilon_{t,q}
    :=
    \left\|
        \widetilde{\boldsymbol W}_t-\boldsymbol P_t
    \right\|_{S_\infty}
    \leq
    \varepsilon_q
    <1,
    \label{eq:NS_error_assumption}
\end{equation}
where $q$ is the number of Newton--Schulz iterations and
$\varepsilon_q$ is independent of $t$.
\end{assumption}

Assumption~\ref{ass:finite_NS_error} is a uniform operator-norm
approximation condition. Through nuclear--spectral norm duality, it
induces a relative loss of alignment proportional to
$\varepsilon_q\|\boldsymbol M_t\|_{S_1}$. Unlike an additive
descent-error condition, it does not introduce a nonvanishing residual
term into the final stationarity bound. Therefore, any fixed number
$q$ of Newton--Schulz iterations preserves convergence to Pareto
stationarity, provided that the uniform bound
\[
    \sup_{t\geq 1}
    \left\|
        \widetilde{\boldsymbol W}_{t,q}-\boldsymbol P_t
    \right\|_{S_\infty}
    \leq \varepsilon_q < 1
\]
holds along the optimization trajectory. In this case, the
finite-step approximation affects only the constants in the
convergence bound.

Moreover, under the standard convergence regime of the employed
Newton--Schulz iteration, the approximation error is nonincreasing
with respect to the number of Newton--Schulz steps and satisfies
\[
    \varepsilon_{q+1}\leq\varepsilon_q,
    \qquad
    \lim_{q\to\infty}\varepsilon_q=0.
\]
Consequently, the effect of the finite-step approximation becomes
weaker as $q$ increases, and the resulting convergence bound approaches
the corresponding exact-polar bound. These properties also explain
why Assumption~\ref{ass:finite_NS_error} is reasonable under the
standard Newton--Schulz convergence conditions.

The following immediate consequence will be used repeatedly.

\begin{lemma}\label{lem:NS_direction_norm}
Under Assumption~\ref{ass:finite_NS_error},
\begin{equation}
    \left\|
        \widetilde{\boldsymbol W}_t
    \right\|_{S_\infty}
    \leq
    1+\varepsilon_{t,q}
    \leq
    1+\varepsilon_q.
\end{equation}
\end{lemma}

\begin{proof}
By the triangle inequality and
\eqref{eq:NS_exact_polar_properties},
\[
\begin{aligned}
    \left\|
        \widetilde{\boldsymbol W}_t
    \right\|_{S_\infty}
    &\leq
    \|\boldsymbol P_t\|_{S_\infty}
    +
    \left\|
        \widetilde{\boldsymbol W}_t-\boldsymbol P_t
    \right\|_{S_\infty} \\
    &\leq
    1+\varepsilon_{t,q}
    \leq
    1+\varepsilon_q.
\end{aligned}
\]
\end{proof}

\begin{lemma}\label{lem:weight_updates_NS}
    Assume that Assumption \ref{ass:bounded_gradient} and \ref{ass:finite_NS_error} hold and that $0 < \beta \gamma <1$. Then for any $t = 1,\dots,T$, the weight updates in Algorithm \ref{alg:moon} satisfy the following inequality
    \begin{equation}
        \| \boldsymbol{z}_{t} - \boldsymbol{z}_{t+1} \|_{\infty} \leq \beta H(1+\varepsilon_q).
    \end{equation}
\end{lemma}

\begin{proof}
The softmax mapping is $1/2$-Lipschitz with respect to the
$\ell_\infty$ norm. Indeed, for any
$\boldsymbol{x},\boldsymbol{y}\in\mathbb{R}^m$,
\[
\left\|
\operatorname{softmax}(\boldsymbol{x})
-
\operatorname{softmax}(\boldsymbol{y})
\right\|_\infty
\le
\frac{1}{2}
\|\boldsymbol{x}-\boldsymbol{y}\|_\infty.
\]
Consequently,
\begin{align}
\|\boldsymbol{z}_{t+1}-\boldsymbol{z}_t\|_\infty
&\le
\frac{1}{2}
\|\boldsymbol{\xi}_{t+1}-\boldsymbol{\xi}_t\|_\infty
\nonumber\\
&=
\frac{\beta}{2}
\|\boldsymbol{\delta}_t+\gamma\boldsymbol{\xi}_t\|_\infty.
\label{eq:softmax_weight_difference_NS}
\end{align}

We first bound $\boldsymbol{\delta}_t$. By the duality between the
spectral norm and the nuclear norm,
\begin{equation}
\begin{aligned}
\|\boldsymbol{\delta}_t\|_\infty
&=
\max_{i\in[m]}
\left|
\left\langle
\widetilde{\boldsymbol{W}}_t,
\nabla\ell^i(\boldsymbol{\Theta}_t)
\right\rangle
\right|
\nonumber\\
&\le
\|\widetilde{\boldsymbol{W}}_t\|_{S_\infty}
\max_{i\in[m]}
\|\nabla\ell^i(\boldsymbol{\Theta}_t)\|_{S_1}\\
&\le (1+\varepsilon_{t,q})H \le (1+\varepsilon_q)H.
\end{aligned}
\label{eq:delta_bound_NS}
\end{equation}

Next, the update of $\boldsymbol{\xi}_t$ can be written as
\[
\boldsymbol{\xi}_{t+1}
=
(1-\beta\gamma)\boldsymbol{\xi}_t
-\beta\boldsymbol{\delta}_t.
\]
Since $0<\beta\gamma<1$, the coefficient
$1-\beta\gamma$ is nonnegative. Thus, using
\eqref{eq:delta_bound_NS},
\begin{align}
\|\boldsymbol{\xi}_{t+1}\|_\infty
&\le
(1-\beta\gamma)\|\boldsymbol{\xi}_t\|_\infty
+\beta\|\boldsymbol{\delta}_t\|_\infty
\nonumber\\
&\le
(1-\beta\gamma)\|\boldsymbol{\xi}_t\|_\infty+\beta (1+\varepsilon_q)H.
\label{eq:xi_recursion_NS}
\end{align}
Unrolling \eqref{eq:xi_recursion_NS} and using
$\boldsymbol{\xi}_1=\boldsymbol{0}$ gives
\begin{align*}
&\|\boldsymbol{\xi}_t\|_\infty
\le
\beta (1+\varepsilon_q)H
\sum_{s=0}^{t-2}(1-\beta\gamma)^s
\nonumber\\
&=
\frac{(1+\varepsilon_q)H}{\gamma}
\left[1-(1-\beta\gamma)^{t-1}\right]
\le
\frac{(1+\varepsilon_q)H}{\gamma}.
\end{align*}

Since $z_t=\operatorname{Softmax}(\boldsymbol{\xi}_t)$, the $1/2$-Lipschitz continuity of softmax under the $\ell_\infty$ norm bounds the change in $z_t$ by the corresponding change in the logits. Applying the triangle inequality to the logit update and substituting $\|\boldsymbol{\delta}_t\|_\infty\leq (1+\varepsilon_q)H$ and $\|\boldsymbol{\xi}_t\|_\infty\leq (1+\varepsilon_q)H/\gamma$, we obtain
\begin{align*}
&\|\boldsymbol{z}_{t+1}-\boldsymbol{z}_t\|_\infty
\le
\frac{\beta}{2}
\left(
\|\boldsymbol{\delta}_t\|_\infty
+\gamma\|\boldsymbol{\xi}_t\|_\infty
\right)\\
&\le
\frac{\beta}{2}((1+\varepsilon_q)H+(1+\varepsilon_q)H)
=
(1+\varepsilon_q)\beta H.
\end{align*}
This completes the proof.
\end{proof}

\begin{lemma}\label{lem:momentum_error_deterministic_NS}
    Assume that Assumption \ref{ass:L-smoothness}, \ref{ass:bounded_gradient} and \ref{ass:finite_NS_error} hold, and that $0 < \beta \gamma <1$. Then for any $t = 1,\dots,T$, the weight updates in Algorithm \ref{alg:moon} satisfy the following inequality 
    \begin{equation}
        \begin{aligned}
            & \|\boldsymbol{M}_t - \sum_{i=1}^m z_{i,t}\nabla\ell^i(\boldsymbol{\Theta}_t)\|_{S_1} \\
            &\leq (1-\mu)^{t-1}\|\sum_{i=1}^m z_{i,1}\nabla\ell^i(\boldsymbol{\Theta}_{1})\|_{S_1} \\
            &+ \frac{L\alpha(1+\varepsilon_q)}{\mu} + \frac{mH^2\beta(1+\varepsilon_q)}{\mu}. 
        \end{aligned}
    \end{equation}
\end{lemma}

\begin{proof}
Let
\[
\boldsymbol{G}_t:=\sum_{i=1}^m z_{i,t}\nabla \ell_i(\boldsymbol{\Theta}_t)
\quad\text{and}\quad
\boldsymbol{E}_t:=\boldsymbol{M}_t-\boldsymbol{G}_t.
\]
Using the momentum update \(\boldsymbol{M}_t=(1-\mu)\boldsymbol{M}_{t-1}+\mu \boldsymbol{G}_t\), we obtain, for \(t\ge 2\),
\[
\boldsymbol{E}_t=(1-\mu)\bigl(\boldsymbol{E}_{t-1}+\boldsymbol{G}_{t-1}-\boldsymbol{G}_t\bigr).
\]
Iterating this recursion and using \(\boldsymbol{M}_0=0\) gives
\begin{align*}
\|\boldsymbol{E}_t\|_{S_1}
&\le
(1-\mu)^t\|\boldsymbol{G}_1\|_{S_1}\\
&+
\sum_{j=2}^t(1-\mu)^{t-j+1}
\|\boldsymbol{G}_{j-1}-\boldsymbol{G}_j\|_{S_1}.
\end{align*}
Moreover,
\begin{align*}
    &\boldsymbol{G}_j-\boldsymbol{G}_{j-1}
    =
    \sum_{i=1}^m z_{i,j}
    \bigl(\nabla\ell_i(\boldsymbol{\Theta}_j)-\nabla\ell_i(\boldsymbol{\Theta}_{j-1})\bigr) \\
    &+
    \sum_{i=1}^m(z_{i,j}-z_{i,j-1})
    \nabla\ell_i(\boldsymbol{\Theta}_{j-1}).
\end{align*}
By Assumptions \ref{ass:L-smoothness}--\ref{ass:bounded_gradient}, Lemma \ref{lem:weight_updates_NS}, and
\(\|\boldsymbol{\Theta}_j-\boldsymbol{\Theta}_{j-1}\|_{S_\infty}=\alpha\|\widetilde{\boldsymbol{W}}_{j-1}\|_{S_\infty}\), we have
\[
\|\boldsymbol{G}_j-\boldsymbol{G}_{j-1}\|_{S_1}
\le
L\alpha(1+\varepsilon_q)+mH^2\beta(1+\varepsilon_q).
\]
Since
\[
\sum_{j=2}^t(1-\mu)^{t-j+1}\le\frac{1}{\mu},
\]
it follows that
\begin{align*}
    &\left\|\boldsymbol{M}_t-\sum_{i=1}^mz_{i,t}\nabla\ell_i(\boldsymbol{\Theta}_t)\right\|_{S_1} \\
    &\le(1-\mu)^{t-1}\left\|\sum_{i=1}^m z_{i,1}\nabla\ell_i(\boldsymbol{\Theta}_1)\right\|_{S_1}\\ &+\frac{L\alpha(1+\varepsilon_q)}{\mu}+\frac{mH^2\beta(1+\varepsilon_q)}{\mu}.
\end{align*}
\end{proof}

\begin{theorem}\label{thm:non_convex_det_NS}
    Suppose Assumption \ref{ass:L-smoothness}--\ref{ass:finite_NS_error} hold, let $\ell^i:\mathbb{R}^{p \times q}\rightarrow\mathbb{R}$ be non-convex for each $i\in [m]$. Suppose that $0<\beta\gamma<1$, then the iterations of Algorithm \ref{alg:moon} satisfy the following inequality 
    \begin{equation}
    \begin{aligned}
        &\frac{1}{T}\sum_{t=1}^T
        \min_{\boldsymbol{z}_t^\ast \in \Delta_m}
        \|\nabla L(\boldsymbol{\Theta}_t)
        \boldsymbol{z}_t^\ast \|_{S_1} \\
        &\leq \frac{1}{1-\varepsilon_q}\bigg(\frac{2B}{\alpha T}+\frac{2H}{\mu T}+\frac{2L\alpha(1+\varepsilon_q)}{\mu} \\
        &+\frac{2mH^2\beta(1+\varepsilon_q)}{\mu}+\frac{mHB\beta(1+\varepsilon_q)}{\alpha} \\
        &+\frac{L\alpha(1+\varepsilon_q)^2}{2}\bigg). 
    \end{aligned}
    \end{equation}
\end{theorem}

\begin{proof}
Define
\[
\boldsymbol{G}_t:=\sum_{i=1}^m z_{i,t}\nabla\ell_i(\boldsymbol{\Theta}_t).
\]
By \(L\)-smoothness and \(\boldsymbol{\Theta}_{t+1}=\boldsymbol{\Theta}_t-\alpha \widetilde{\boldsymbol{W}}_t\),
\begin{align*}
\sum_{i=1}^m z_{i,t}\ell_i(\boldsymbol{\Theta}_{t+1})
&\le
\sum_{i=1}^m z_{i,t}\ell_i(\boldsymbol{\Theta}_t)
-\alpha\langle \boldsymbol{G}_t,\widetilde{\boldsymbol{W}}_t\rangle\\
&\quad+\frac{L\alpha^2(1+\varepsilon_q)^2}{2}.
\end{align*}
Since \(\widetilde{\boldsymbol{W}}_t\) is the direction returned by the finite-step
Newton–Schulz iteration,
\(\langle \boldsymbol{M}_t,\widetilde{\boldsymbol{W}}_t\rangle=\langle \boldsymbol{M}_t,\boldsymbol{P}_t\rangle+\langle\boldsymbol{M}_t,\widetilde{\boldsymbol{W}}_t-\boldsymbol{P}_t\rangle\ge(1-\varepsilon_q)\|\boldsymbol{M}_t\|_{S_1}\). Therefore,
\begin{align*}
&-\alpha\langle \boldsymbol{G}_t,\widetilde{\boldsymbol{W}}_t\rangle
\le
-\alpha\langle \boldsymbol{M}_t,\widetilde{\boldsymbol{W}}_t\rangle + \alpha\langle\boldsymbol{M}_t-\boldsymbol{G}_t,\widetilde{\boldsymbol{W}}_t\rangle\\
&\le 
-(1-\varepsilon_q)\alpha\|\boldsymbol{M}_t\|_{S_1} + (1+\varepsilon_q)\alpha\|\boldsymbol{M}_t-\boldsymbol{G}_t\|_{S_1}\\
&\le -(1-\varepsilon_q)\alpha(\|\boldsymbol{G}_t\|_{S_1}-\|\boldsymbol{M}_t-\boldsymbol{G}_t\|_{S_1}) \\
&\quad+ (1+\varepsilon_q)\alpha\|\boldsymbol{M}_t-\boldsymbol{G}_t\|_{S_1}\\
&\le -(1-\varepsilon_q)\alpha\|\boldsymbol{G}_t\|_{S_1} + 2\alpha\|\boldsymbol{M}_t-\boldsymbol{G}_t\|_{S_1}. 
\end{align*}
After adding and subtracting
\(\sum_i z_{i,t+1}\ell_i(\boldsymbol{\Theta}_{t+1})\) and rearranging, we obtain
\[
\begin{aligned}
&(1-\varepsilon_q)\|\boldsymbol{G}_t\|_{S_1} \\
&\le {}
\frac{1}{\alpha}
\left(
\sum_{i=1}^m z_{i,t}\ell_i(\boldsymbol{\Theta}_t)
-
\sum_{i=1}^m z_{i,t+1}\ell_i(\boldsymbol{\Theta}_{t+1})
\right)\\
&+\frac{1}{\alpha}
\sum_{i=1}^m
(z_{i,t+1}-z_{i,t})\ell_i(\boldsymbol{\Theta}_{t+1}) \\
&+\frac{L\alpha(1+\varepsilon_q)^2}{2} +2\|\boldsymbol{M}_t-\boldsymbol{G}_t\|_{S_1}.
\end{aligned}
\]
Lemma~\ref{lem:weight_updates_NS} imply
\[
\left|
\sum_{i=1}^m
(z_{i,t+1}-z_{i,t})\ell_i(\boldsymbol{\Theta}_{t+1})
\right|
\le mBH\beta(1+\varepsilon_q).
\]
Applying Lemma~\ref{lem:momentum_error_deterministic_NS}, summing over \(t=1,\ldots,T\), and using
\(\|\boldsymbol{G}_1\|_{S_1}\le H\), we obtain
\[
\begin{aligned}
&(1-\varepsilon_q)\sum_{t=1}^T\|\boldsymbol{G}_t\|_{S_1}
\le
\frac{2B}{\alpha}
+\frac{2H}{\mu}
+\frac{2TL\alpha(1+\varepsilon_q)}{\mu} \\
&+\frac{2TmH^2\beta(1+\varepsilon_q)}{\mu}+\frac{TmHB\beta(1+\varepsilon_q)}{\alpha} \\
&+\frac{TL\alpha(1+\varepsilon_q)^2}{2}.
\end{aligned}
\]
Finally,
\[
\min_{z\in\Delta_m}
\left\|
\sum_{i=1}^m z_i\nabla\ell_i(\boldsymbol{\Theta}_t)
\right\|_{S_1}
\le \|\boldsymbol{G}_t\|_{S_1}.
\]
Dividing the preceding inequality by \((1-\varepsilon_q)T\) yields the claimed result.
\end{proof}

There is an immediate corollary of the above Theorem. 

\begin{corollary}\label{cor:convergence_rate_nonconvex_det_NS}
    Setting $\alpha=\sqrt{\frac{B}{LT}}$, $\beta=\frac{1}{mHT}$ and $0<\gamma<mHT$ in \cref{thm:non_convex_det_NS}, then, for any fixed $\mu \in (0,1]$, the following holds 
    \begin{equation}
        \frac{1}{T}\sum_{t=1}^T \min_{\boldsymbol{z}_t^\ast \in \Delta_m}\|\nabla L(\boldsymbol{\Theta_t})\boldsymbol{z}_t^\ast \|_{S_1} \leq \frac{1}{1-\varepsilon_q}O(1/\sqrt{T}). 
    \end{equation}
\end{corollary}

\begin{proof}
    Since $\alpha=\sqrt{\frac{B}{LT}}$, $\beta=\frac{1}{mHT}$ and $0<\gamma<mHT$ in Theorem \ref{thm:non_convex_det_NS}, we have 
    \begin{align*}
        &\frac{1}{T}\sum_{t=1}^T
        \min_{\boldsymbol{z}_t^\ast \in \Delta_m}
        \|\nabla L(\boldsymbol{\Theta_t})
        \boldsymbol{z}_t^\ast \|_{S_1} \\
        &\leq \frac{1}{1-\varepsilon_q}\bigg(\frac{2B}{\alpha T}+\frac{2H}{\mu T}+\frac{2L\alpha(1+\varepsilon_q)}{\mu} \\
        &\quad+\frac{2mH^2\beta(1+\varepsilon_q)}{\mu}+\frac{mHB\beta(1+\varepsilon_q)}{\alpha} \\
        &\quad+\frac{L\alpha(1+\varepsilon_q)^2}{2}\bigg) \\
        &= \frac{1}{1-\varepsilon_q}\bigg[\bigg(3+\frac{2(1+\varepsilon_q)}{\mu} + (1+\varepsilon_q) \\
        &\quad+\frac{(1+\varepsilon_q)^2}{2}\bigg)\sqrt{\frac{LB}{T}} + \frac{(4+2\varepsilon_q)H}{\mu T} \bigg] \\ 
        &=\frac{1}{1-\varepsilon_q}\mathcal{O}(1/\sqrt{T}). 
    \end{align*}
\end{proof}

Increasing $q$ improves all three
approximation-dependent factors in Theorem~8:
\[
    \frac{1}{1-\varepsilon_q},
    \qquad
    1+\varepsilon_q,
    \qquad
    (1+\varepsilon_q)^2.
\]
Hence, the effect of the finite-step approximation becomes
weaker as more Newton--Schulz steps are performed, and the
bound continuously approaches the exact-polar convergence
bound as $q\to\infty$. For any fixed $q$ satisfying the
uniform condition $\varepsilon_q<1$, the approximation changes
only the constants and does not alter the
$\mathcal O(T^{-1/2})$ deterministic convergence rate.

\section{A Two-Task Matrix Example}
\label{app:two_task_example}

We present a two-task example to illustrate the difference between Euclidean gradient manipulation and the matrix-aware update derived in Proposition~\ref{prop:moon_dual}. Let the parameter be $\Theta\in\mathbb{R}^{2\times2}$ and define
\begin{equation}
\begin{aligned}
\ell_1(\Theta)&=\frac{1}{2}\|\Theta\|_F^2+\langle A+B,\Theta\rangle,\\
\ell_2(\Theta)&=\frac{1}{2}\|\Theta\|_F^2+\langle A-B,\Theta\rangle,
\end{aligned}
\end{equation}
where
\begin{equation}
A=\begin{pmatrix}2&0\\0&1\end{pmatrix},\qquad B=\begin{pmatrix}0&1\\-1&0\end{pmatrix}.
\end{equation}

At $\Theta_0=0$, the two losses have the same value, $\ell_1(\Theta_0)=\ell_2(\Theta_0)=0$, and their gradients are
\begin{equation}
\begin{aligned}
g_1&=\nabla\ell_1(\Theta_0)=A+B
=\begin{pmatrix}2&1\\-1&1\end{pmatrix},\\
g_2&=\nabla\ell_2(\Theta_0)=A-B
=\begin{pmatrix}2&-1\\1&1\end{pmatrix}.
\end{aligned}
\end{equation}

For a task weight $z\in[0,1]$, let $s=2z-1$. The weighted gradient is
\begin{equation}
\begin{aligned}
G(z)&=zg_1+(1-z)g_2\\
&=A+sB
=\begin{pmatrix}2&s\\-s&1\end{pmatrix}.
\end{aligned}
\end{equation}

\paragraph{MGDA direction.}
MGDA selects the task weight by minimizing the squared Frobenius norm:
\begin{equation}
\frac{1}{2}\|G(z)\|_F^2=\frac{1}{2}\left(5+2s^2\right).
\end{equation}
The unique minimizer is $s^\ast=0$, or equivalently $z^\ast=1/2$, and hence $G_{\mathrm{MGDA}}^\ast=A$. To compare the two methods under the same spectral-norm trust-region budget, we normalize the MGDA direction to unit spectral norm:
\begin{equation}
\begin{aligned}
&P_{\mathrm{MGDA}}=\frac{G_{\mathrm{MGDA}}^\ast}{\|G_{\mathrm{MGDA}}^\ast\|_{S_\infty}}
=\begin{pmatrix}1&0\\0&\frac{1}{2}\end{pmatrix},\\
&\|P_{\mathrm{MGDA}}\|_{S_\infty}=1.
\end{aligned}
\end{equation}

\paragraph{MOON direction.}
MOON instead minimizes the squared nuclear norm. For any $2\times2$ matrix $G$,
\begin{equation}
\|G\|_{S_1}^2=\|G\|_F^2+2|\det(G)|.
\end{equation}
Since $\det(G(z))=2+s^2>0$, the MOON objective becomes
\begin{equation}
\begin{aligned}
\frac{1}{2}\|G(z)\|_{S_1}^2
&=\frac{1}{2}\left(5+2s^2+2(2+s^2)\right)\\
&=\frac{1}{2}\left(9+4s^2\right).
\end{aligned}
\end{equation}
Its unique minimizer is also $s^\ast=0$, so the optimal weighted gradient is
\begin{equation}
G_{\mathrm{MOON}}^\ast=A=\begin{pmatrix}2&0\\0&1\end{pmatrix}.
\end{equation}
The singular values of $G_{\mathrm{MOON}}^\ast$ are $2$ and $1$. Therefore, its polar factor is
\begin{equation}
\begin{aligned}
&P_{\mathrm{MOON}}=\operatorname{Polar}(G_{\mathrm{MOON}}^\ast)=UV^\top=I,\\
&\|P_{\mathrm{MOON}}\|_{S_\infty}=1.
\end{aligned}
\end{equation}
Because the singular values of $G_{\mathrm{MOON}}^\ast$ are unequal, $P_{\mathrm{MGDA}}$ and $P_{\mathrm{MOON}}$ are not scalar multiples of each other. Thus, their difference cannot be attributed solely to update scaling.

\paragraph{Common descent comparison.}
We compare the two directions under the same unit spectral-norm budget. For MGDA,
\begin{equation}
\begin{aligned}
\langle g_1,P_{\mathrm{MGDA}}\rangle
&=\langle g_2,P_{\mathrm{MGDA}}\rangle\\
&=2+\frac{1}{2}=\frac{5}{2}.
\end{aligned}
\end{equation}
For MOON,
\begin{equation}
\begin{aligned}
\langle g_1,P_{\mathrm{MOON}}\rangle
&=\langle g_2,P_{\mathrm{MOON}}\rangle\\
&=2+1=3.
\end{aligned}
\end{equation}
Consequently,
\begin{equation}
\min_{i\in\{1,2\}}\langle g_i,P_{\mathrm{MOON}}\rangle >\min_{i\in\{1,2\}}\langle g_i,P_{\mathrm{MGDA}}\rangle.
\end{equation}
Hence, for the parameter update $\Theta^+=\Theta_0-\eta P$, MOON provides a strictly larger first-order decrease for both objectives while using the same spectral-norm update budget:
\begin{equation}
\|\eta P_{\mathrm{MGDA}}\|_{S_\infty}
=\|\eta P_{\mathrm{MOON}}\|_{S_\infty}
=\eta.
\end{equation}

\paragraph{Quadratic upper-bound comparison.}
For either objective,
\begin{equation}
\begin{aligned}
\|\nabla\ell_i(\Theta)-\nabla\ell_i(\Theta')\|_{S_1}
&=\|\Theta-\Theta'\|_{S_1}\\
&\leq2\|\Theta-\Theta'\|_{S_\infty}.
\end{aligned}
\end{equation}
Thus, both losses are $2$-smooth with respect to the spectral norm. For either direction satisfying $\|P\|_{S_\infty}=1$,
\begin{equation}
\ell_i(\Theta_0-\eta P)
\leq\ell_i(\Theta_0)
-\eta\langle g_i,P\rangle+\eta^2.
\end{equation}
The corresponding upper bounds are
\begin{equation}
\begin{aligned}
\ell_i(\Theta_0-\eta P_{\mathrm{MGDA}})
&\leq-\frac{5}{2}\eta+\eta^2,\\
\ell_i(\Theta_0-\eta P_{\mathrm{MOON}})
&\leq-3\eta+\eta^2.
\end{aligned}
\end{equation}
Therefore, the MOON direction gives a strictly smaller worst-case quadratic upper bound for every $\eta>0$.

The exact objective values are
\begin{equation}
\begin{aligned}
\ell_i(\Theta_0-\eta P_{\mathrm{MGDA}})
&=-\frac{5}{2}\eta+\frac{5}{8}\eta^2,\\
\ell_i(\Theta_0-\eta P_{\mathrm{MOON}})
&=-3\eta+\eta^2.
\end{aligned}
\end{equation}
It follows that
\begin{equation}
\begin{aligned}
\ell_i(\Theta_0-\eta P_{\mathrm{MOON}})
<&\ \ell_i(\Theta_0-\eta P_{\mathrm{MGDA}}),\\
& i\in\{1,2\},\quad 0<\eta<\frac{4}{3}.
\end{aligned}
\end{equation}
This example demonstrates that the advantage of MOON is not merely caused by a larger update magnitude. Under the same spectral-norm budget, its matrix-aware direction provides a strictly stronger common descent projection, a tighter worst-case quadratic upper bound, and a larger exact decrease for a nontrivial range of step sizes.

\section{Experimental Details \& Further Experiments}
\label{section:experimental_details}

\subsection{Benchmarks and Baselines}

We consider five benchmarks widely used in multi-task learning research: \textbf{MultiMNIST} \citep{sabour2017dynamic}, \textbf{NYU-v2} \citep{silberman2012indoor}, \textbf{CityScapes} \citep{Cordts2016Cityscapes}, \textbf{QM9} \citep{blum2009970} and \textbf{CelebA} \citep{liu2015deep}. Specifically, \textbf{NYU-v2} is an indoor RGB-D dataset, comprising 1449 aligned RGB and depth images with pixel-level semantic annotations over 13 classes. In NYU-v2 benchmark, we consider image segmentation, depth prediction and surface normal prediction as three tasks. \textbf{CityScapes} is an urban street-scene benchmark with 5000 finely annotated stereo RGB image pairs providing dense per-pixel semantic labels for scene understanding. \textbf{MultiMNIST} is an MTL variant of MNIST dataset, where each image contains two randomly sampled MNIST digits. In MultiMNIST, we treat classifying each digit as a separate task. \textbf{QM9} is a quantum-chemistry benchmark of over 130k small organic molecules with DFT-computed 3D geometries, widely used for molecular property prediction and multi-task learning. The objective is to predict 11 molecular properties. \textbf{CelebA} is a large-scale face attribute dataset containing over 200k facial images annotated with 40 binary attributes. We treat the prediction of each attribute as an individual task, resulting in a 40-task multi-task classification problem.

We compare the proposed method with 13 multi-objective optimization baselines: (1) \textbf{Single-task learning (STL)}, training each task separately; (2) \textbf{Linear Scalarization (LS)}, assigning equal weights to all tasks and minimizing the resulting weighted loss; (3) a \textbf{scale-invariant} baseline \textbf{(SI)} that applies equal weighting to the log losses; (4) \textbf{Random Loss Weighting (RLW)} \citep{lin2021closer}, sampling task weights from a normal distribution; (5) \textbf{Dynamic Weight Average (DWA)} \citep{liu2019end}, adjusting task weights according to the rates of loss change; (6) \textbf{Uncertainty Weighting (UW)} \citep{kendall2018multi}, incorporating task uncertainty to determine loss weights; (7) \textbf{MGDA} \citep{sener2018multi}, finding a common descent direction that yields balanced improvement across tasks at each step; (8) \textbf{GradNorm} \citep{chen2018gradnorm}, dynamically balancing task losses by adjusting their weights according to relative training rates and gradient magnitudes;(9) \textbf{PCGrad} \citep{yu2020gradient}, projecting the gradient of each task to that of other tasks before aggregation; (10) \textbf{CAGrad} \citep{liu2021conflict}, minimizing the average loss while controlling the minimum decrease across tasks; (11) \textbf{IMTL-G} \citep{liu2021towards}, learning task weights such that the aggregated gradient has equal projection onto each task gradient; (12) \textbf{Nash-MTL} \citep{navon2022multi}, formulating multi-objective optimization as a bargaining game and finding a solution that benefits all tasks; (13) \textbf{FAMO} \citep{liu2024famo}, maximizing the worst-case improvement rate and using log losses to update task weights.

\subsection{Comparison with Naive Combinations with Muon}
\label{app:muon_combination}

To examine whether MOON can be replaced by Muon itself or by directly combining existing multi-objective optimization methods with Muon, we additionally evaluate \textbf{Muon} by using it as the optimizer under standard linear scalarization, as Muon itself can already achieve competitive MTL performance without an explicit MOO mechanism. We further construct two baselines, \textbf{MGDA+Muon} and \textbf{FAMO+Muon}, by retaining the original multi-objective gradient aggregation of MGDA and FAMO while replacing Adam with Muon for parameter updates. We evaluate all methods on MultiMNIST with the ViT backbone under the same setting as the main experiments.

\begin{table}[h]
\centering
\begin{small}
\begin{sc}
\setlength{\tabcolsep}{5pt}
\begin{tabular}{lccc}
\toprule
\textbf{Method} & \textbf{Left acc} & \textbf{Right acc} & \textbf{Average acc} \\
\midrule
MGDA        & 95.58 & 94.10 & 94.84 \\
MGDA+Muon   & 94.61 & 93.65 & 94.13 \\
FAMO        & 95.89 & 94.88 & 95.39 \\
FAMO+Muon   & 95.61 & 94.33 & 94.97 \\
Muon        & 95.39 & 94.82 & 95.11 \\
\cmidrule{1-4}
MOON (Ours) & \textbf{95.99} & \textbf{95.31} & \textbf{95.65} \\
\bottomrule
\end{tabular}
\end{sc}
\end{small}
\caption{Performance comparison on MultiMNIST. Per-task test accuracies (\%) and their average are reported. Each experiment is repeated over 3 random seeds and the mean is reported.}
\label{tab:muon_combination}
\end{table}

As shown in Table~\ref{tab:muon_combination}, Muon itself also achieves competitive performance under linear scalarization, but still falls short of MOON. Replacing Adam with Muon degrades the performance of both MGDA and FAMO, while MOON achieves the best results. This shows that simply combining traditional multi-objective methods with Muon is insufficient. MOON instead formulates multi-objective optimization directly over matrix-valued parameters, jointly accounting for objective trade-offs and matrix update geometry.

\subsection{Transformer Experiment on CelebA}
\label{app:celeba_vit}

To further evaluate MOON on a larger dataset with a Transformer architecture, we conduct experiments on CelebA using a ViT backbone. CelebA contains 40 binary attribute prediction tasks, providing a substantially larger multi-objective setting than MultiMNIST. We compare MOON with representative multi-objective optimization methods under the same experimental setting.

\begin{table}[h]
\centering
\begin{small}
\begin{sc}
\begin{tabular}{lc}
\toprule
Method & $\Delta m\% \downarrow$ \\
\midrule
MGDA        & 19.20 \\
PCGrad      & 10.48 \\
CAGrad      & 12.61 \\
IMTL-G      & 7.96 \\
Nash-MTL    & 8.06 \\
FAMO        & 7.75 \\
\cmidrule{1-2}
MOON (Ours) & \textbf{6.09} \\
\bottomrule
\end{tabular}
\end{sc}
\end{small}
\caption{Performance comparison on CelebA with a ViT backbone. Lower $\Delta m\%$ indicates better overall multi-task performance. Each experiment is repeated over 3 random seeds and the mean is reported.}
\label{tab:celeba_vit}
\end{table}

As shown in Table~\ref{tab:celeba_vit}, MOON achieves the lowest $\Delta m\%$ among all compared methods. This demonstrates that MOON remains effective on a larger multi-task dataset with a Transformer architecture and scales well to a setting with 40 objectives.

\subsection{Multi-Objective Reinforcement Fine-Tuning on Qwen3-1.7B-Base}
\label{app:qwen3_math500}

To evaluate MOON beyond conventional supervised multi-task learning, we apply it to multi-objective reinforcement fine-tuning of Qwen3-1.7B-Base on MATH500, following the setup of Figure~1(c) in \citet{lu2026uncovering}. The 500 problems are split into 300 training, 100 validation, and 100 test examples. The training split is used exclusively for policy optimization, while the validation split is used for monitoring, hyperparameter tuning, and checkpoint selection under a prespecified selection rule. The test split is withheld during training and model selection and is used only to evaluate the validation-selected checkpoint. Accordingly, Figure~\ref{fig:qwen3_math500} reports validation learning curves, whereas final performance is evaluated on the held-out test split.

REINFORCE is used to jointly optimize three alignment objectives: accuracy, conciseness, and clarity. All three objectives are implemented as verifiable binary rewards during training. Accuracy receives a reward of one when the final answer extracted from the last valid \texttt{\textbackslash boxed\{\}} expression matches the ground-truth answer after standard MATH normalization, and zero otherwise. Conciseness receives a reward of one when the generated response is shorter than the running global average response length, and zero otherwise. Clarity receives a reward of one when the response contains a nonempty boxed answer and a step-by-step structure indicated by at least two enumerated steps or ordinal transitions, and zero otherwise. During evaluation, accuracy and clarity are reported as mean binary reward scores, while conciseness is reported as the mean response length in tokens, with lower values preferred.

We compare MOON with Linear, GradNorm, MGDA, Nash-MTL, and FAMO. For MOON, the objective-specific policy-gradient signals are aggregated through the proposed matrix-aware multi-objective update to optimize the shared LLM parameters.

As shown by the validation trajectories in Figure~\ref{fig:qwen3_math500}, MOON rapidly reaches a favorable trade-off across the three objectives. Its clarity score rises faster than those of most baselines, while accuracy remains competitive and response length decreases steadily. On the held-out test set, the validation-selected MOON checkpoint maintains high clarity and competitive accuracy while producing substantially shorter responses. These results demonstrate that MOON remains effective for multi-objective reinforcement fine-tuning of a modern LLM and reaches a strong multi-objective solution with relatively few optimization steps.

\subsection{Ablation on Key Hyperparameters}
\label{app:hyperparameter_ablation}

We study the sensitivity of MOON to two key hyperparameters, the task-weight update stepsize $\beta$ and the weight decay coefficient $\gamma$, on MultiMNIST with the ViT backbone. We vary one hyperparameter at a time while keeping the remaining settings unchanged. Each experiment is repeated over 3 random seeds and the mean is reported.

\begin{table}[h]
\centering
\begin{small}
\begin{sc}
\setlength{\tabcolsep}{4pt}
\begin{tabular}{lccccc}
\toprule
$\gamma$ & $10^{-2}$ & $5{\times}10^{-2}$ & $10^{-3}$ & $5{\times}10^{-3}$ & $10^{-4}$ \\
\midrule
Average Acc & 95.14 & 95.33 & \textbf{95.65} & 95.62 & 95.28 \\
\bottomrule
\end{tabular}
\end{sc}
\end{small}
\caption{Average accuracy (\%) on MultiMNIST with different values of $\gamma$.}
\label{tab:gamma_ablation}
\end{table}

\begin{table}[h]
\centering
\begin{small}
\begin{sc}
\setlength{\tabcolsep}{5pt}
\begin{tabular}{lccccc}
\toprule
$\beta$ & $10^{-5}$ & $5{\times}10^{-5}$ & $10^{-4}$ & $5{\times}10^{-4}$ & $10^{-3}$ \\
\midrule
Average Acc & 95.11 & 95.53 & \textbf{95.65} & 95.48 & 94.93 \\
\bottomrule
\end{tabular}
\end{sc}
\end{small}
\caption{Average accuracy (\%) on MultiMNIST with different values of $\beta$.}
\label{tab:beta_ablation}
\end{table}

As shown in Tables~\ref{tab:gamma_ablation} and~\ref{tab:beta_ablation}, MOON maintains stable performance across a range of $\beta$ and $\gamma$, indicating that it is not sensitive to moderate variations of these hyperparameters.

\subsection{Ablation on Key Components}

\paragraph{Orthonormalized Updates.} We conduct an ablation study on the MultiMNIST experiment to evaluate the impact of the orthonormalized updates. We compare MOON with that using the Euclidean-direction updates while retaining the matrix-norm-oriented task weighting.

As shown in Table~\ref{tab:ablation_components}, without the orthonormalized update, using only matrix-norm-oriented task weighting leads to inferior performance. This indicates that the orthonormalized updates play an important role in the effectiveness of MOON.

\paragraph{Momentum.} We further conduct an ablation study on the momentum mechanism by directly constructing the update from the current aggregated gradient without momentum smoothing.

As shown in Table~\ref{tab:ablation_components}, removing momentum results in a clear performance degradation. This indicates that the momentum is practically necessary.

\begin{table}[h]
\centering
\begin{tabular}{lccc}
\toprule
Method & Left & Right & Average \\
\midrule
MOON & \textbf{95.99} & \textbf{95.31} & \textbf{95.65} \\
\quad (with Euclidean update) & 95.18 & 94.42 & 94.80 \\
\quad (without momentum) & 93.96 & 93.44 & 93.70 \\
\bottomrule
\end{tabular}
\caption{Ablation on the orthonormalized updates and momentum on MultiMNIST. Test accuracies (\%) for the left and right digit classification tasks and their average are reported.}
\label{tab:ablation_components}
\end{table}

\subsection{30-task Synthetic Data Regression Experiment}

\begin{figure}[h]
\begin{center}
\centering
\includegraphics[width=0.6\columnwidth]{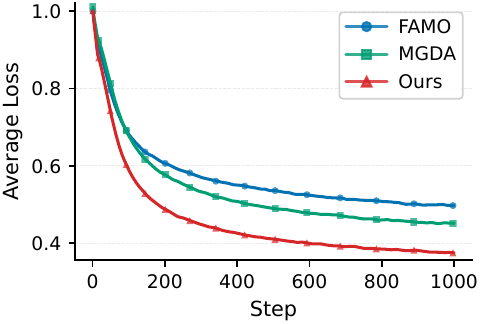}
\caption{Average squared loss after 1000 training steps on the synthetic data multi-objective regression problem.}
\Description{Line chart comparing the average squared loss of MGDA, FAMO, and MOON over 1000 steps. MOON converges faster and ends with the lowest loss.}
\label{fig:toy_avgloss}
\end{center}
\end{figure}

To better understand steepest descent for matrix-valued parameters, we construct a multi-objective regression problem on synthetic data. We adopt an MLP with parameters including weight matrices of its hidden linear layers. We randomly generate 20-dimensional inputs and 30-dimensional targets, treating the regression of each target dimension as a separate task. The average square loss $\boldsymbol{L} = \frac{1}{nd}\sum_{i=1}^{n}\|\boldsymbol{y}_i-\boldsymbol{\hat{y}}_i\|_{2}^{2}$ across these tasks is used to demonstrate training efficiency, where $n$ represents the sample size and $d$ represents the output dimension. We compare the proposed MOON method with traditional Euclidian-based gradient manipulation methods like MGDA and FAMO.

As shown in Figure \ref{fig:toy_avgloss}, MOON converges faster and attains a lower average training loss than Euclidean-based gradient manipulation baselines. We observe that in the early stage, MOON quickly opens a clear gap over MGDA and FAMO, and this advantage persists throughout training. After 1k steps, MOON reduces the final average loss by about 20\% compared to FAMO and about 15\% to MGDA.

\begin{figure}[t!]
\centering
\begin{minipage}[t]{0.32\textwidth}
  \centering
  \includegraphics[width=\linewidth]{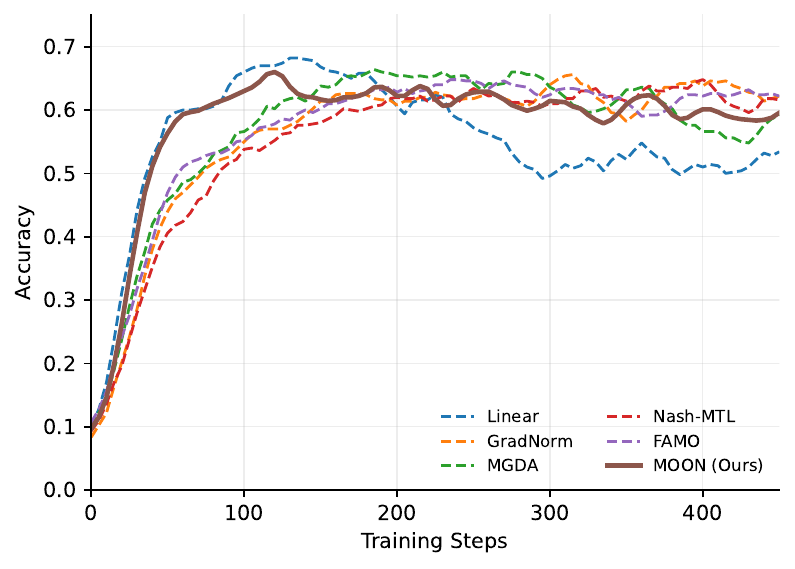}\par
  \small (a) Accuracy
\end{minipage}\hfill
\begin{minipage}[t]{0.32\textwidth}
  \centering
  \includegraphics[width=\linewidth]{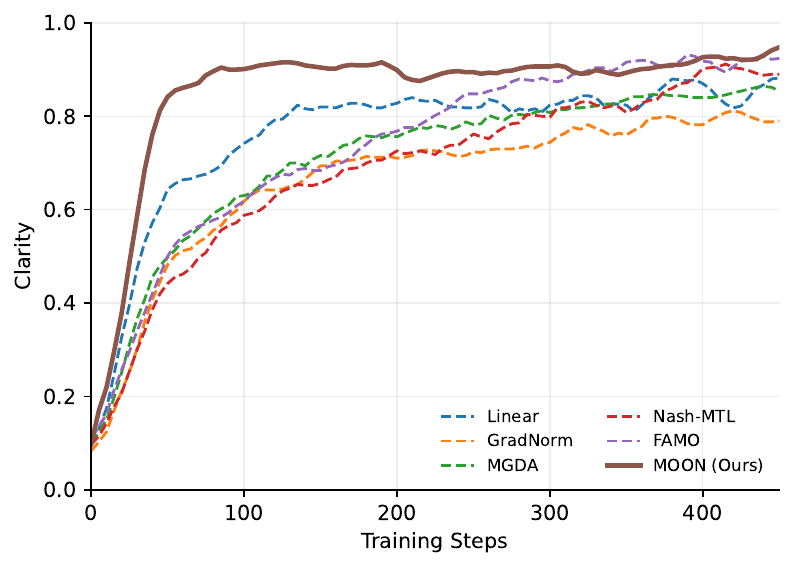}\par
  \small (b) Clarity
\end{minipage}\hfill
\begin{minipage}[t]{0.32\textwidth}
  \centering
  \includegraphics[width=\linewidth]{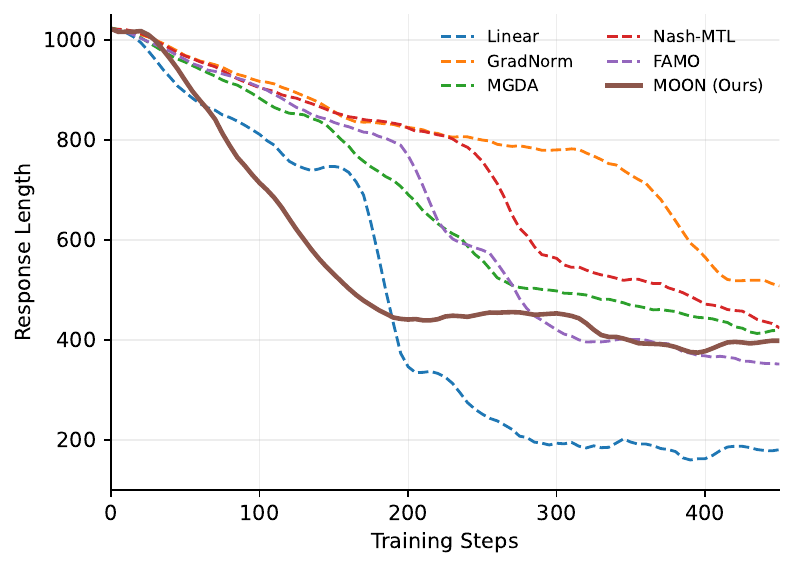}\par
  \small (c) Response length
\end{minipage}

\caption{Multi-objective reinforcement fine-tuning of Qwen3-1.7B-Base on MATH500 following the setting of Figure~1(c) in \citet{lu2026uncovering}. We report accuracy, clarity, and response length throughout training, where response length serves as the measure of conciseness. Higher accuracy and clarity are better, while shorter responses indicate better conciseness.}
\Description{Three line charts compare Linear, GradNorm, MGDA, Nash-MTL, FAMO, and MOON during multi-objective reinforcement fine-tuning of Qwen3-1.7B-Base on MATH500. MOON improves clarity rapidly while maintaining competitive accuracy and reducing response length, reaching a favorable trade-off among the three objectives in fewer training steps.}
\label{fig:qwen3_math500}
\end{figure}

\subsection{Training Efficiency Comparison}

We evaluate the training efficiency of MOON on MultiMNIST by tracking the average training cross-entropy (CE) loss against wall-clock time, with MGDA and FAMO as baselines. As reported in Table~\ref{tab:training_ce_loss_time}, MOON attains a lower training loss under the same time budget throughout training. The gap widens at later stages: after 1,000 seconds, the average CE loss reaches 0.144 with MOON, compared with 0.215 for MGDA and 0.179 for FAMO.

\begin{table}[H]
\centering
\begin{small}
\begin{sc}
\begin{tabular}{lccc}
\toprule
Time (s) & MGDA & FAMO & MOON \\
\midrule
200  & 0.500 & 0.368 & \textbf{0.362} \\
400  & 0.339 & 0.277 & \textbf{0.259} \\
600  & 0.273 & 0.222 & \textbf{0.199} \\
800  & 0.229 & 0.188 & \textbf{0.164} \\
1000 & 0.215 & 0.179 & \textbf{0.144} \\
\bottomrule
\end{tabular}
\end{sc}
\end{small}
\caption{Average training CE loss vs. training time.}
\label{tab:training_ce_loss_time}
\end{table}

\begin{table}[H]
\centering
\begin{small}
\begin{sc}
\begin{tabular}{lc}
\toprule
Method & Time (s) \\
\midrule
MGDA & 1576 \\
FAMO & 1105 \\
MOON & \textbf{949} \\
\bottomrule
\end{tabular}
\end{sc}
\end{small}
\caption{Training time required to reach an average CE loss of 0.15.}
\label{tab:time_to_ce_015}
\end{table}

We also measure the wall-clock time required to reduce the average training CE loss below a common threshold. As shown in Table~\ref{tab:time_to_ce_015}, MOON reaches an average CE loss of 0.15 approximately 39.8\% faster than MGDA and 14.1\% faster than FAMO. Thus, the additional computation introduced by MOON is offset by faster convergence, resulting in lower end-to-end training time.

We further compare GPU memory usage on MultiMNIST. MGDA, FAMO, and MOON consume 2,076 MB, 2,076 MB, and 2,078 MB, respectively. This indicates that MOON exhibits no significant difference in memory overhead compared with the other MOO methods.

Overall, MOON improves wall-clock convergence while retaining essentially the same memory footprint as MGDA and FAMO.

\FloatBarrier
\subsection{Experimental Results with Error Bars}

In the following tables \ref{table:multimnist_std} to \ref{table:qm9_std}, we report the results of MOON with error bars.

\begin{table}[H]
\centering
\begin{tabular}{@{}lccc@{}}
\toprule
Method
& Left Acc $\uparrow$
& Right Acc $\uparrow$
& Average Acc $\uparrow$ \\
\midrule
MOON (mean)
& \textbf{95.99}
& \textbf{95.31}
& \textbf{95.65} \\
MOON (std)
& $\pm 0.07$
& $\pm 0.10$
& $\pm 0.07$ \\
\bottomrule
\end{tabular}
\caption{Test accuracy (\%) on MultiMNIST (2 tasks). Each experiment is repeated over 3 random seeds. The mean and the standard deviation is reported. Per-task accuracies (left/right) and their average are reported. The best result is marked in bold.}
\label{table:multimnist_std}
\end{table}

\begin{table}[H]
\centering
{\small
\setlength{\tabcolsep}{2.5pt}
\resizebox{\textwidth}{!}{%
\begin{tabular}{@{}lcccccccccc@{}}
\toprule
& \multicolumn{2}{c}{Segmentation}
& \multicolumn{2}{c}{Depth}
& \multicolumn{5}{c}{Surface Normal}
& \\
\cmidrule(lr){2-3}
\cmidrule(lr){4-5}
\cmidrule(lr){6-10}
Method
& mIoU $\uparrow$
& Pix Acc $\uparrow$
& Abs Err $\downarrow$
& Rel Err $\downarrow$
& Mean $\downarrow$
& Median $\downarrow$
& $11.25^\circ \uparrow$
& $22.5^\circ \uparrow$
& $30^\circ \uparrow$
& $\Delta m\% \downarrow$ \\
\midrule
MOON (mean)
& 39.41
& \textbf{67.03}
& \textbf{0.4891}
& \textbf{0.2091}
& 25.61
& 20.27
& 28.89
& 54.87
& 67.32
& \textbf{-4.63} \\
MOON (std)
& $\pm 0.48$
& $\pm 0.24$
& $\pm 0.0024$
& $\pm 0.0030$
& $\pm 0.11$
& $\pm 0.14$
& $\pm 0.19$
& $\pm 0.23$
& $\pm 0.22$
& $\pm 0.22$ \\
\bottomrule
\end{tabular}
}
}
\caption{Results on NYU-v2 (3 tasks) dataset. Each experiment is repeated over 3 random seeds. The mean and the standard deviation is reported. The best average result is marked in bold. The task specific metrics and the average performance drop $\boldsymbol{\Delta m}\%$ are reported.}
\label{table:nyuv2_std}
\end{table}

\begin{table}[H]
\centering
\begin{tabular}{@{}lcccccc@{}}
\toprule
& \multicolumn{5}{c}{CityScapes}
& CelebA \\
\cmidrule(lr){2-6}
\cmidrule(lr){7-7}
& \multicolumn{2}{c}{Segmentation}
& \multicolumn{2}{c}{Depth}
& $\Delta m\% \downarrow$
& $\Delta m\% \downarrow$ \\
\cmidrule(lr){2-3}
\cmidrule(lr){4-5}
Method
& mIoU $\uparrow$
& Pix Acc $\uparrow$
& Abs Err $\downarrow$
& Rel Err $\downarrow$
& & \\
\midrule
MOON (mean)
& \textbf{78.61}
& \textbf{94.36}
& \textbf{0.0126}
& 31.41
& \textbf{1.54}
& \textbf{4.65} \\
MOON (std)
& $\pm 0.24$
& $\pm 0.09$
& $\pm 0.0008$
& $\pm 1.00$
& $\pm 0.60$
& $\pm 0.08$ \\
\bottomrule
\end{tabular}
\caption{Results on CityScapes (2 tasks) and CelebA (40 tasks) dataset. Each experiment is repeated over 3 random seeds. The mean and the standard deviation is reported. The best average result is marked in bold. Task-specific metrics and the average performance drop $\boldsymbol{\Delta m}\%$ are reported.}
\label{table:city_celeba_std}
\end{table}

\begin{table}[H]
\centering
{\small
\setlength{\tabcolsep}{2.2pt}
\begin{tabular}{@{}lcccccccccccc@{}}
\toprule
Method
& $\mu$
& $\alpha$
& $\epsilon_{\mathrm{HOMO}}$
& $\epsilon_{\mathrm{LUMO}}$
& $\langle R^2 \rangle$
& ZPVE
& $U_0$
& $U$
& $H$
& $G$
& $c_v$
& $\Delta m\% \downarrow$ \\
\cmidrule(lr){2-12}
& \multicolumn{11}{c}{MAE $\downarrow$}
& \\
\midrule
MOON (mean)
& \textbf{0.07}
& \textbf{0.24}
& \textbf{52.9}
& \textbf{71.0}
& 3.35
& 5.67
& \textbf{46.4}
& \textbf{46.7}
& \textbf{46.9}
& \textbf{46.3}
& \textbf{0.08}
& \textbf{49.9} \\
MOON (std)
& $\pm 0.006$
& $\pm 0.009$
& $\pm 2.8$
& $\pm 2.3$
& $\pm 0.035$
& $\pm 0.110$
& $\pm 1.9$
& $\pm 2.0$
& $\pm 1.9$
& $\pm 2.0$
& $\pm 0.004$
& $\pm 1.60$ \\
\bottomrule
\end{tabular}
}
\caption{Results on QM9 dataset (11 tasks). Each experiment is repeated over 3 random seeds. The mean and the standard deviation is reported. The best average result is marked in bold. The task specific metrics and the average performance drop $\boldsymbol{\Delta m}\%$ are reported.}
\label{table:qm9_std}
\end{table}

\FloatBarrier
\section{Limitations}
Our theoretical analysis assumes the exact computation of the polar factor, whereas our implementation uses a Newton--Schulz approximation. In practice, the approximation error is small, we therefore omit this error from the current analysis and leave its precise theoretical characterization to future work.

\end{document}